\documentclass{article}
\PassOptionsToPackage{numbers}{natbib}
\usepackage[final]{neurips_2023}
\usepackage[utf8]{inputenc}
\usepackage[T1]{fontenc}
\usepackage{pifont}
\usepackage{caption}
\usepackage[nohints]{minitoc}
\usepackage{nicefrac}
\usepackage{microtype}
\usepackage{amsmath,amsfonts,bm}

\newcommand{\newterm}[1]{{\bf #1}}

\def\Figref#1{Figure~\ref{#1}}

\def\eqref#1{equation~\ref{#1}}
\def\Eqref#1{Equation~\ref{#1}}
\def\1{\bm{1}}

\def\rmR{{\mathbf{R}}}

\def\vzero{{\bm{0}}}

\def\vlambda{{\bm{\lambda}}}

\def\ve{{\bm{e}}}

\def\vj{{\bm{j}}}

\def\vu{{\bm{u}}}
\def\vv{{\bm{v}}}

\def\vx{{\bm{x}}}
\def\vy{{\bm{y}}}

\def\evlambda{{\lambda}}

\def\evu{{u}}

\def\evx{{x}}

\def\mA{{\bm{A}}}
\def\mB{{\bm{B}}}

\def\mD{{\bm{D}}}

\def\mI{{\bm{I}}}
\def\mJ{{\bm{J}}}

\def\mL{{\bm{L}}}
\def\mM{{\bm{M}}}

\def\mP{{\bm{P}}}
\def\mQ{{\bm{Q}}}

\def\mS{{\bm{S}}}

\def\mU{{\bm{U}}}
\def\mV{{\bm{V}}}
\def\mW{{\bm{W}}}
\def\mX{{\bm{X}}}
\def\mY{{\bm{Y}}}
\def\mZ{{\bm{Z}}}

\def\mLambda{{\bm{\varLambda}}}

\DeclareMathAlphabet{\mathsfit}{\encodingdefault}{\sfdefault}{m}{sl}
\SetMathAlphabet{\mathsfit}{bold}{\encodingdefault}{\sfdefault}{bx}{n}

\def\gG{{\mathcal{G}}}

\def\gX{{\mathcal{X}}}

\def\sE{{\mathbb{E}}}

\def\sV{{\mathbb{V}}}

\def\emA{{A}}

\def\emD{{D}}

\newcommand{\E}{\mathbb{E}}

\newcommand{\R}{\mathbb{R}}

\DeclareMathOperator*{\argmax}{arg\,max}

\usepackage{cite}
\usepackage{amsmath}
\usepackage{amsfonts}
\usepackage{amssymb}
\usepackage{amsthm}
\usepackage{mathtools}
\usepackage{upgreek}
\usepackage{graphicx}
\usepackage{xcolor}
\usepackage{wrapfig}
\usepackage{pgfplots}
\pgfplotsset{compat=1.15}
\usepackage{mathrsfs}
\usetikzlibrary{arrows}
\usepackage{booktabs}
\usepackage{tabularx}
\newcolumntype{Y}{>{\centering\arraybackslash}X}
\usepackage{makecell}
\usepackage{multirow}
\usepackage{diagbox}
\usepackage{bigstrut}
\usepackage{rotating}
\usepackage{algorithm}
\usepackage{algpseudocode}
\usepackage{url}
\usepackage{multicol}
\definecolor{mydarkblue}{rgb}{0,0.08,0.45}
\usepackage[colorlinks=true,
    linkcolor=mydarkblue,
    citecolor=mydarkblue,
    filecolor=mydarkblue,
    urlcolor=mydarkblue]{hyperref}

\newtheorem{theorem}{Theorem}
\newtheorem{definition}{Definition}
\newtheorem{corollary}{Corollary}
\newtheorem{lemma}{Lemma}
\newtheorem{assumption}{Assumption}
\theoremstyle{definition}\newtheorem*{remark}{Remark}
\DeclareMathOperator{\diag}{diag}

\DeclareMathOperator{\NN}{NN}

\DeclareMathOperator{\rank}{rank}

\title{
    Laplacian Canonization: A Minimalist Approach to Sign and Basis Invariant
    Spectral Embedding
}

\author{%
    Jiangyan Ma\textsuperscript{1$*$} \qquad Yifei Wang\textsuperscript{2}\thanks{Equal Contribution.} \qquad Yisen Wang\textsuperscript{3,4}\thanks{Corresponding Author: Yisen Wang (yisen.wang@pku.edu.cn)} \\ 
    \textsuperscript{1} School of Electronics Engineering and Computer Science, Peking University \\
    \textsuperscript{2} School of Mathematical Sciences, Peking University\\  
    \textsuperscript{3} National Key Lab of General Artificial Intelligence, \\ School of Intelligence Science and Technology, Peking University\\
    \textsuperscript{4} Institute for Artificial Intelligence, Peking University\\
}

\begin{document}

    \maketitle
    \doparttoc
    \faketableofcontents

    \begin{abstract}
        Spectral embedding is a powerful graph embedding technique that has received
        a lot of attention recently due to its effectiveness on Graph Transformers.
        However, from a theoretical perspective, the universal expressive power of
        spectral embedding comes at the price of losing two important invariance
        properties of graphs, sign and basis invariance, which also limits its
        effectiveness on graph data. To remedy this issue, many previous methods
        developed costly approaches to learn new invariants and suffer from high
        computation complexity. In this work, we explore a minimal approach that
        resolves the ambiguity issues by directly finding canonical directions for the
        eigenvectors, named Laplacian Canonization (LC). As a pure pre-processing
        method, LC is light-weighted and can be applied to any existing GNNs. We
        provide a thorough investigation, from theory to algorithm, on this approach,
        and discover an efficient algorithm named Maximal Axis Projection (MAP) that
        works for both sign and basis invariance and successfully canonizes more than
        90\% of all eigenvectors. Experiments on real-world benchmark datasets like
        ZINC, MOLTOX21, and MOLPCBA show that MAP consistently outperforms existing
        methods while bringing minimal computation overhead. Code is available at
        \url{https://github.com/PKU-ML/LaplacianCanonization}.
    \end{abstract}

    \section{Introduction}

    Despite the popularity of Graph Neural Networks (GNNs) for graph representation
    learning \citep{gcn,sgc,dissecting,g2cn,gind}, it is found that many existing GNNs have limited expressive power and
    cannot tell the difference between many non-isomorphic graphs
    \citep{gin,higher-order}. Existing approaches to improve expressive power, such
    as high-order GNNs \citep{higher-order} and subgraph aggregation \citep{esan},
    often incur significant computation costs over vanilla GNNs, limiting their use in
    practice. In comparison, a simple but effective strategy is to use discriminative
    node identifiers. For example, GNNs with random features can lead to universal
    expressive power \citep{random-features,rni}. However, these unique node
    identifiers often lead to the loss of the permutation invariance property of GNNs,
    which is an important inductive bias of graph data that matters for  sample
    complexity and generalization \citep{universal-invariant-equivariant-gnn}.
    Therefore, during the pursuit of more expressive GNNs, we should also maintain
    the invariance properties of graph data. Balancing these two conflicting demands
    presents a significant challenge for the development of advanced GNNs.

    Spectral embedding (SE) is a classical approach to encode node positions using
    eigenvectors $\mU$ of the Laplacian matrix $\mL$, which has the advantage of
    being expressive and permutation equivariant. GNNs using SE can attain universal
    expressive power (distinguishing \emph{any} pair of non-isomorphic graphs) even
    under simple architectures (results in Section~\ref{sec:universal}). However,
    spectral embedding faces two additional challenges in preserving graph invariance:
    sign ambiguity and basis ambiguity, due to the non-uniqueness of
    eigendecomposition. These ambiguities could lead to inconsistent predictions for
    the same graph under different eigendecompositions. 

    Many methods have been proposed to address sign and basis ambiguities of spectral
    embedding. A popular heuristic is RandSign \citep{benchmarking-gnn} which randomly
    flips the signs of the eigenvectors during training. Although it is simple to use,
    this data augmentation approach does not offer any formal invariance guarantees
    and can lead to slower convergence due to all possible $2^n$ sign flips. An
    alternative involves using sign-invariant eigenfunctions to attain sign invariance
    \citep{lpe}, which significantly increases the time complexity to $\mathcal{O}(n^4)$.
    Another solution is to design extra GNN modules for sign and basis invariant
    embeddings, \textit{e.g.}, SignNet and BasisNet \citep{signnet}, which can also
    add a substantial computational burden. Therefore, as summarized in
    Table~\ref{tab:comparison}, existing spectral embedding methods are all
    detrimental in a certain way that either hampers sign and basis invariance, or
    induces large computational overhead. More discussions about the related work
    can be found in Appendix \ref{app:related}. 
   
    \begin{table}[t]
        \centering
        \caption{Comparison between prior works and our method.
        $n$ is the number of nodes, $m$ is the exponent of the
        feature dimension of BasisNet \citep{signnet,invariant-universal}.}
        \resizebox{\textwidth}{!}{
            \begin{tabularx}{1.08\textwidth}{Xcccccc}
                \toprule
                Method         & \thead{pre-processing\\ time} & universality   & \thead{permutation\\ invariance} & \thead{addresses sign\\ ambiguity} & \thead{addresses basis\\ ambiguity} & \thead{feature\\ dimension} \\
                \midrule
                LapPE \citep{benchmarking-gnn} & $\mathcal{O}(n^3)$ & \ding{56}      & \ding{52}      & \ding{56}      & \ding{56}      & $n$ \\
                RandSign \citep{benchmarking-gnn} & $\mathcal{O}(n^3)$ & \ding{56}      & \ding{52}      & \ding{52}      & \ding{56}      & $n$ \\
                SAN \citep{lpe} & $\mathcal{O}(n^4)$ & \ding{52}      & \ding{52}      & \ding{52}      & \ding{56}      & $3n$ \\
                SignNet \citep{signnet} & $\mathcal{O}(n^3)$ & \ding{52}      & \ding{52}      & \ding{52}      & \ding{56}      & $2n$ \\
                BasisNet \citep{signnet} & $\mathcal{O}(n^3)$ & \ding{52}      & \ding{52}      & \ding{52}      & \ding{52}      & $n^m$ \\
                MAP (ours)     & $\mathcal{O}(n^3)$ & \ding{52}      & \ding{52}      & \ding{52}      & \ding{52}      & $n$ \\
                \bottomrule
            \end{tabularx}
        }
        \label{tab:comparison}
    \end{table}

    In this work, we explore a new approach called Laplacian Canonization (LC)
    that resolves the ambiguities by identifying a unique canonical direction for
    each eigenvector, amongst all its sign and basis equivalents. Although it is
    relatively easy to find a canonization rule that work for certain vectors, up to
    now, there still lacks a rigorous understanding of what kinds of vectors are
    canonizable and whether we could find a complete algorithm for all canonizable
    features. In this paper, we systematically answer this problem by developing a
    general theory for Laplacian canonization and characterizing the sufficient and
    necessary conditions for sign and basis canonizable features. 
    
    Based on these theoretical properties, we propose a practical canonization
    algorithm for sign and basis invariance, called Maximal Axis Projection (MAP),
    that adopts the permutation-invariant axis projection functions to determine the
    canonical directions. We theoretically characterize the conditions under which
    MAP can guarantee sign and basis canonization, and empirically verify that this
    condition holds for most synthetic and real-world graphs. It is worth noting
    that LC is a lightweight approach since it is only a pre-processing method and
    does not alter the dimension of the spectral embedding. Empirically, we show that
    employing the MAP-canonized spectral embedding yields significant improvements
    over the vanilla RandSign approach, and even matches SignNet on large-scale
    benchmark datasets like OGBG \citep{ogb}. We summarize our contributions as
    follows:
    
    \begin{itemize}
        \item We explore Laplacian Canonization (LC), a new approach to restoring
        the sign and basis invariance of spectral embeddings via determining the
        canonical direction of eigenvectors in the pre-processing stage. We develop
        a general theoretical framework for LC and characterize the canonizability
        of sign and basis invariance.
        \item We propose an efficient algorithm for Laplacian Canonization, named
        Maximal Axis Projection (MAP), that works well for both sign and basis
        invariance. In particular, we show that MAP-sign is capable of canonizing
        all sign canonizable eigenvectors and thus is complete. The assessment of its
        feasibility shows that MAP can effectively canonize almost all eigenvectors on
        random graphs and more than 90\% eigenvectors on real-world datasets.
        \item We evaluate the MAP-canonized spectral embeddings on graph
        classification benchmarks including ZINC, MOLTOX21 and MOLPCBA, and obtain
        consistent improvements over previous spectral embedding methods while
        inducing the smallest computational overhead.
    \end{itemize}

    \section{Benefits and Challenges of Spectral Embedding for GNNs}
    \label{sec:universal}

    Denote a graph as $\gG=(\sV,\sE,\mX)$ where $\sV$ is the vertex set of size $n$,
    $\sE$ is the edge set, and $\mX\in\R^{n\times d}$ are the input node features.
    We denote $\mA$ as the adjacency matrix, and let
    \(
        \hat{\mA}=\mD^{-\frac12}\tilde{\mA}\mD^{-\frac12}
        =\mD^{-\frac12}(\mI+\mA)\mD^{-\frac12}
    \)
    be the normalized adjacency matrix, where $\mI$ denotes the augmented self-loop,
    $\mD$ is the diagonal degree matrix of $\tilde{\mA}$ defined by $\emD_{i,i}=
    \sum_{j=1}^n\tilde{\emA}_{i,j}$. A graph function $f([\mX,\hat{\mA}])$ is
    \emph{permutation invariant} if for all permutation matrix $\mP\in\R^{n\times n}$,
    we have $f([\mP\mX,\mP\hat{\mA}\mP^\top])=f([\mX,\hat{\mA}])$. Similarly, $f$ is
    \emph{permutation equivariant} if $f([\mP\mX,\mP\hat{\mA}\mP^\top ])=\mP f([\mX,
    \hat{\mA}])$.

    \textbf{Spectral Embedding (SE).} Considering the limited expressive power of
    MP-GNNs, recent works explore more flexible GNNs like graph Transformers
    \citep{graphormer,gtn,gt,lpe,gps}. These models bypass the explicit structural
    inductive bias in MP-GNNs while encoding graph structures via positional
    embedding (PE). A popular graph PE is spectral embedding (SE), which uses the
    eigenvectors $\mU$ of the Laplacian matrix $\mL=\mI-\hat\mA$ with
    eigendecomposition $\mL=\mU\mLambda\,\mU^\top$, where $\mLambda=\diag(\vlambda)$
    is the diagonal matrix of ascending eigenvalues $\evlambda_1\leq\dots\leq
    \evlambda_n$, and the $i$-th column of $\mU$ is the eigenvector corresponding to
    $\evlambda_i$. It is easy to see that spectral embedding is permutation
    equivariant: for any node permutation $\mP$ of the graph, $\mP\mU$ is the
    spectral embedding of the new Laplacian since $\mP\mL\mP^\top =(\mP\mU)\mLambda
    (\mP\mU)^\top$. Therefore, a permutation-invariant GNN (\textit{e.g.}, DeepSets
    \citep{deep-sets}, GIN \citep{gin}, Graph Transformer \citep{gt}) using the
    SE-augmented input features $\tilde{\mX}=[\mX,\mU]$ remains permutation invariant.

    \textbf{Reweighted Spectral Embedding (RSE).} Previous works have shown that using
    spectral embedding improves the expressive power of MP-GNNs \citep{dgn}. Nonetheless,
    we find that SE alone is \emph{insufficient} to approximate an arbitrary graph
    function, as it does not contain all information about the graph, in particular,
    the eigenvalues. Consider two non-isomorphic graphs whose Laplacian matrices have
    identical eigenvectors but different eigenvalues. As their SE is the same, a
    network only using SE cannot tell them apart. To resolve this issue, we propose
    \newterm{reweighted spectral embedding (RSE)} that additionally reweights
    each eigenvector $\vu_i$ with the square-root of its corresponding eigenvalue
    $\lambda_i$, \textit{i.e.}, $\mU_\mathrm{RSE}=\mU\mLambda^\frac12$. With the
    reweighting technique, RSE incorporates eigenvalue information without
    need extra dimensions.
    
    \textbf{Universality of RSE\@.} In the following theorem, we prove that with RSE,
    \emph{any} universal network on sets (\textit{e.g.}, DeepSets \citep{deep-sets}
    and Transformer \citep{transformer-universal}) is universal on graphs while
    preserving permutation invariance. All proofs are deferred to
    Appendix~\ref{app:proofs}.

    \begin{theorem}\label{thm:se-universal}
        Let\/ $\Omega\subset\R^{n\times d}\times\R^{n\times n}$ be a compact set
        of graphs, $[\mX,\hat{\mA}]\in\Omega$. Let\/ $\NN$ be a universal neural
        network on sets. Given any continuous invariant graph function $f$ defined
        over\/ $\Omega$ and arbitrary $\varepsilon>0$, there exist a set of NN
        parameters such that for all graphs $[\mX,\hat{\mA}]\in\Omega$,
        \[
            |\NN([\mX,\mU_\mathrm{RSE}])-f([\mX,\hat{\mA}])|<\varepsilon.
        \]
    \end{theorem}

    As far as we know, this theorem is the first to show that, with the help of graph
    embeddings like RSE, even an MLP network like DeepSets \citep{deep-sets}
    (composed of a node-wise MLP, a global pooling layer, and a graph-level MLP) can
    achieve universal expressive power to distinguish any pair of non-isomorphic
    graphs. Notably, it does not violate the NP-hardness of graph isomorphism testing,
    since training NN itself is known to be NP-hard \citep{np-complete}. Actually, it
    is not always necessary to use all spectra of the graph. Existing studies find
    that high-frequency components are often unhelpful, or even harmful for
    representation learning \citep{spectral-expressive,gfnn}. Thus, in practice, we
    only use the first $k$ low-frequency components of RSE. An upper bound on the
    approximation error of truncated RSE can be found in Appendix~\ref{app:bound}.

    \textbf{Ambiguities of eigenvectors.} Although RSE enables universal GNNs, there
    exist two types of ambiguity in eigenvectors that hinder their applications. The
    first one, known as \newterm{sign ambiguity}, arises when an eigenvector $\vu_
    {\lambda_i}$ corresponding to eigenvalue $\lambda_i$ is equally valid with its
    sign flipped, \textit{i.e.}, $-\vu_{\lambda_i}$ is also an eigenvector
    corresponding to $\lambda_i$. The second one, termed \newterm{basis ambiguity},
    occurs when eigenvalues with multiplicity degree $d_i>1$ can have any other
    orthogonal basis in the subspace spanned by the corresponding eigenvectors as
    valid eigenvectors. To be specific, for multiple eigenvalues with multiplicity
    degree $d_i>1$, the corresponding eigenvectors $\mU_{\lambda_i}\in\R^{n\times d_i}$
    form an orthonormal basis of a subspace. Then any orthonormal matrix $\mQ\in\R^
    {d_i\times d_i}$ can be applied to $\mU_{\lambda_i}$ to generate a new set of
    valid eigenvectors for $\lambda_i$. Because of these ambiguities, we can get
    distinct GNN outputs for the same graph, resulting in unstable and suboptimal
    performance \citep{benchmarking-gnn,lpe,lspe,signnet}. In Appendix~\ref{app:why},
    we elaborate the challenges posed by these ambiguities.

    \section{Laplacian Canonization for Sign and Basis Invariance}

    Rather than incorporating additional modules to learn new sign and basis
    invariants \citep{lpe,signnet}, we explore a straightforward, learning-free
    approach named Laplacian Canonization (LC). The general idea of LC is to
    determine a unique direction for each eigenvector $\vu$ among all its sign and
    basis equivalents. In this way, the ambiguities can be directly addressed in the
    pre-processing stage. For example, for two sign ambiguous vectors $\vu$ and
    $-\vu$, we aim to find an algorithm that determines a unique direction among
    them to obtain sign invariance. Despite some na\"ive canonization rules
    (discussed in Appendix \ref{app:assumption1}), there still lacks a systematical
    understanding of the following key questions of LC: 
    \begin{enumerate}
        \item What kind of canonization algorithm should we look for?
        (Section~\ref{sec:definition})
        \item What kind of eigenvectors are canonizable or non-canonizable?
        (Section~\ref{sec:canonical})
        \item Is there an efficient canonization algorithm for all canonizable
        features? (Section~\ref{sec:pratical})
    \end{enumerate}
    In this section, we answer the three problems by establishing the first formal
    theory of Laplacian canonization and characterizing the canonizability for sign
    and basis invariance; based on these analyses, we also propose an efficient
    algorithm named MAP for LC that is guaranteed to canonize all sign canonizable
    features. To get a glimpse of our final results, Table~\ref{tab:violation} shows
    that MAP can resolve both sign and basis ambiguities for more than 90\% of all
    eigenvectors on real-world data.

    \begin{table}[t]
        \centering
        \caption{The ratio of uncanonizable eigenvectors \emph{w.r.t.}\ each invariance
        property with our MAP algorithm on three real-world datasets: ZINC, MOLTOX21,
        and MOLPCBA\@.}
        \begin{tabular}{cccc}
            \toprule
            Invariance        & ZINC          & MOLTOX21      & MOLPCBA \\
            \midrule
            Sign              & 2.46\,\%      & 3.04\,\%      & 2.24\,\% \\
            Basis             & 1.59\,\%      & 3.31\,\%      & 7.37\,\% \\
            Total             & 4.05\,\%      & 6.35\,\%      & 9.61\,\% \\
            \bottomrule
        \end{tabular}
        \label{tab:violation}
    \end{table}

    \subsection{Definition of Canonization}\label{sec:definition}

    To begin with, we first find out what properties a desirable canonization
    algorithm should satisfy. The ultimate goal of canonization is to eliminate
    \emph{ambiguities} by selecting a \emph{canonical form} among these ambiguous
    outputs. Generally speaking, we can characterize ambiguity as a multivalued
    function $f\colon\mathcal{X}\to\mathcal{Y}$, \textit{i.e.}, there could be
    multiple outputs $y_1,\dots,y_n$ corresponding to the same input $x$. In our
    sign/basis ambiguity case, $f$ refers to a mapping from a graph $x\in\gX$ to the
    eigenvectors of a certain eigenvalue via the ambiguous eigendecomposition. The
    ambiguous eigenvectors are not independent; they are related by a sign/basis
    transformation. In general, we can assume that all possible outputs $y_1,\dots,
    y_n$ of any $x$ belong to the same equivalence class induced by some group action
    $g\in G$, where $G$ acts on $\mathcal{Y}$. That is, if $f(x)=y_1=y_2$, then there
    exists $g\in G$ such that $y_1=gy_2$. Moreover, eigenvectors of graphs obey a
    fundamental symmetry: permutation equivariance. In general, we assume that $f$
    is equivariant to a group $H$ acting on $\mathcal{X}$ and $\mathcal{Y}$. That is,
    for any $h\in H$, if $y_1,\dots,y_n$ are all possible outputs of $x$, then $hy_1,
    \dots,hy_n$ are all possible outputs of $hx$.

    Specifically, for the goal of canonizing ambiguous eigenvectors, \textit{i.e.},
    Laplacian canonization, we are interested in algorithms invariant to sign/basis
    transformations in the corresponding orthogonal group $G$ ($O(1)$ for sign
    invariance and $O(d)$ for basis invariance). Meanwhile, to preserve the symmetry
    of graph data, this algorithm should also maintain the permutation equivariant
    property of eigenvectors (Section~\ref{sec:universal}) \textit{w.r.t.}\ the
    permutation group $H$. Third, it also should still be discriminative enough to
    produce different canonical forms for different graphs (like the original
    spectral embedding). Combining these desiderata, we have a formal definition of
    canonization.

    \begin{definition}\label{def:canonization}
        A mapping $\mathcal{A}\colon\mathcal{Y}\to\mathcal{Y}$ is called a
        ($f,G,H$)-\textbf{canonization} when it satisfies:
        \begin{itemize}
            \item $\mathcal{A}$ is \textbf{$G$-invariant}:
            $\forall y\in\mathcal{Y},g\in G$, $\mathcal{A}(y)=\mathcal{A}(gy)$;
            \item $\mathcal{A}$ is \textbf{$H$-equivariant}:
            $\forall x\in\mathcal{X}, h\in H$, $\mathcal{A}\bigl(f(hx)\bigr)
            =h\mathcal{A}\bigl(f(x)\bigr)$;
            \item $\mathcal{A}$ is \textbf{universal}:
            $\forall x\in\mathcal{X}, h\in H$, $x\neq hx\Rightarrow\mathcal{A}
            \bigl(f(x)\bigr)\neq\mathcal{A}\bigl(f(hx)\bigr)$.
        \end{itemize}        
        Accordingly, for any $x\in\mathcal{X}$, if there exists a canonization
        $\mathcal{A}$ for $x$, we say $x$ is \textbf{canonizable}.
    \end{definition}
    
    \subsection{Theoretical Properties of Canonization}\label{sec:canonical}

    Following Definition~\ref{def:canonization}, we are further interested in
    the question that whether any eigenvector $\vu$ is canonizable, \textit{i.e.,},
    there exists a canonization that can determine its unique direction. Unfortunately,
    the answer is NO\@. For example, the vector $\vu=(1,-1)$ cannot be canonized by
    any canonization algorithm, since a permutation of $\vu$ gives $(-1,1)$ that
    equals to $-\vu$\footnote{In this case, sign invariance conflicts with permutation
    equivariance: making $\vu$ and $-\vu$ output the same vector violates permutation
    equivariance, and vice versa.}. But as long as there is only a small number of
    uncanonizable eigenvectors like $\vu$, a canonization algorithm can still resolve
    the ambiguity issue to a large extent. 

    Therefore, we are interested in the fundamental question of which eigenvectors
    are canonizable. The following theorem establishes a general necessary and
    sufficient condition of the canonizability for general groups $G,H$, which
    may be of independent interest.

    \begin{theorem}\label{thm:canonizable}
        An input $x\in\mathcal{X}$ is canonizable on the embedding function $f$
        \textit{iff} there does not exist $h\in H$ and $g\in G$ such that $x\neq hx$
        and $f(hx)=gf(x)$.
    \end{theorem}

    This theorem states that for inputs from the same equivalence class in $\mathcal{X}$
    (induced by $H$), as long as they are not mapped to the same equivalence class in
    $\mathcal{Y}$ (induced by $G$), these inputs are canonizable. In
    particular, by applying Theorem~\ref{thm:canonizable} to the specific group $G$
    induced by sign/basis invariance, we can derive some simple rules to determine
    whether there exists a canonizable rule for given eigenvector(s).

    \begin{corollary}[Sign canonizability]\label{cor:sign-canonizable}
        A vector $\vu\in\R^n$ is canonizable under sign ambiguity
        \textit{iff} there does not exist a permutation matrix $\mP\in\R^{n\times n}$
        such that $\vu=-\mP\vu$.
    \end{corollary}

    \begin{corollary}[Basis canonizability]\label{cor:basis-canonizable}
        The base eigenvectors $\mU\in\R^{n\times d}$ of the eigenspace $V$ are
        canonizable under basis ambiguity \textit{iff} there does not exist
        a permutation matrix $\mP\in\R^{n\times n}$ such that $\mU\neq\mP\mU$ and\/
        $\mathop{\mathrm{span}}(\mP\mU)=\mathop{\mathrm{span}}(\mU)=V$.
    \end{corollary}

    \begin{remark}
        From a probabilistic view, almost all eigenvectors are canonizable. Let
        $\mU\in\R^{n\times d}$ be basis vectors sampled from a continuous distribution
        in $\R^n$, then $\Pr\{\mU\text{ is canonizable}\}=1$.
    \end{remark}

    In the next subsection, we will further design an efficient algorithm to
    canonize all sign and basis canonizable eigenvectors in the pre-processing stage,
    so the network does not have to bear the ambiguities of these eigenvectors.

    \subsection{MAP: A Practical Algorithm for Laplacian Canonization}
    \label{sec:pratical}

    Built upon theoretical properties in Section~\ref{sec:canonical}, we aim to design
    a \emph{general, powerful, and efficient} canonization to resolve both sign and
    basis ambiguities for as many eigenvectors as possible. Here, we choose to adopt
    axis projection as the basic operator in our canonization algorithm named Maximal
    Axis Projection (MAP). The key observation is that the standard basis vectors
    (\textit{i.e.}, the axis) of the Euclidean space $\ve_i\in\R^n$ are permutation
    equivariant, and in the meantime, the eigenspace spanned by the eigenvectors
    $V=\mathop{\mathrm{span}}(\mU)$ are also permutation equivariant. This means that
    when projecting the axis to the eigenspace, the obtained angles are also
    permutation equivariant, based on which we could apply permutation invariant
    functions (such as $\max$) to obtain permutation invariant statistics that can be
    used for canonization. Meanwhile, due to the generality of projection, it can be
    applied to both sign ambiguity (for a single eigenvector) and basis ambiguity
    (for the eigenspace with multiple eigenvectors). We provide illustrative examples
    to help understand this algorithm in Appendix~\ref{app:toy}.

    \textbf{Preparation step: Axis projection.} Consider unit eigenvector(s)
    $\mU\in\R^{n\times d}$ corresponding to an eigenvalue $\lambda$ with geometric
    multiplicity $d\geq 1$. These eigenvectors span a $d$-dimensional eigenspace
    $V=\mathop{\mathrm{span}}(\mU)$ with the projection matrix $\mP=\mU\mU^\top$. 
    To start with, we calculate the projected angle between $V$ and each standard
    basis (\textit{i.e.}, the axis) of the Euclidean space $\ve_i$ (a one-hot vector
    whose $i$-th element is $1$), \textit{i.e.}, $\alpha_i=|\mP\ve_i|,i=1,\dots,n$.
    Assume that there are $k$ distinct values in $\{\alpha_i,i=1,\dots,n\}$, according
    to which we can divide all basis vectors $\{\ve_i\}$ into $k$ disjoint groups
    $\mathcal{B}_i$ (arranged in descending order of the distinct angles). Each
    $\mathcal{B}_i$ represents an equivalent class of axes that $\vu_i$ has the same
    projection on. Then, we define a summary vector $\vx_i$ for the axes in each group
    $\mathcal{B}_i$ as their total sum $\vx_i=\sum_{\ve_j\in\mathcal{B}_i}\ve_j+c$,
    where $c$ is a tunable constant. Next, we introduce how to adopt these summary
    vectors to canonize the eigenvectors for sign and basis invariance, respectively.

    \subsubsection{Sign Canonization with MAP}\label{sec:sign}
    
    \textbf{Step 1. Find non-orthogonal axis.} For sign canonization, we calculate the
    angles between the eigenvalue $\vu$ and each summary vector $\vx_i$ one by one,
    and terminate the procedure as long as we find a summary vector $\vx_h$ with
    non-zero angle $\alpha_h=\vu^\top \vx_h\neq0$, and return \texttt{<none>}
    otherwise.
    \begin{equation}\label{eq:find-non-orthogonal}
        \vx_h=
        \begin{cases}
            \min_{i}\varPhi,&\text{if }\varPhi\neq\varnothing,\\
            \texttt{<none>},&\text{if }\varPhi=\varnothing,
        \end{cases}
        \text{ where }\varPhi=\{i\mid\vu^\top\vx_i\neq 0\}.
    \end{equation}

    \begin{assumption}\label{ass:mild}
        There exists a summary vector $\vx_h$ that is non-orthogonal to $\vu$,
        \textit{i.e.}, $\vu^\top\vx_h\neq 0$,
    \end{assumption}

    \textbf{Step 2. Sign canonization.} As long as Assumption~\ref{ass:mild} holds,
    we can utilize $\vx_h$ to determine the canonical direction $\vu^*$ by requiring
    its projected angle to be positive, \textit{i.e.},
    \begin{equation}\label{eq:sign-canonization}
        \vu^*=
        \begin{cases}
            \vu,&\text{if }\vu^\top\vx_h>0,\\
            -\vu,&\text{if }\vu^\top\vx_h<0.
        \end{cases}
    \end{equation}

    We call this algorithm MAP-sign, and summarize it in Algorithm~\ref{alg:sign}
    in Appendix~\ref{app:pseudo-code}. The following theorem proves that it yields a
    valid canonization for sign invariance under Assumption~\ref{ass:mild}.

    \begin{theorem}\label{thm:sign}
        Under Assumption~\ref{ass:mild}, our MAP-sign algorithm gives a sign-invariant,
        permutation-equivariant and universal canonization of $\vu$.
    \end{theorem}

    One would wonder how restrictive the non-orthogonality condition
    (Assumption~\ref{ass:mild}) is for sign canonization. In the following theorem,
    we establish a strong result showing that sign canonizability is equivalent to
    non-orthogonality. In other words, \emph{any sign canonizable eigenvectors can be
    canonized by MAP-sign}. Due to this equivalence, MAP-sign can also serve as a
    complete algorithm to determine sign canonizability: a vector $\vu$ is sign
    canonizable iff $\vx_h$ is not \texttt{<none>} in \eqref{eq:find-non-orthogonal}.
    
    \begin{theorem}\label{thm:up-to-permutation}
        A vector $\vu\in\R^n$ is sign canonizable \textit{iff} there exists a summary
        vector $\vx_h$ s.t.\ $\vu^\top\vx_h\neq 0$.
    \end{theorem}

    \begin{remark}
        Theorem~\ref{thm:sign} is proved in Appendix~\ref{app:sign} and experimentally
        verified in Appendix~\ref{app:verify-sign}. We also show that MAP is not the
        only complete canonization algorithm for sign sinvariance by proposing another
        polynomial-based algorithm in Appendix~\ref{app:alternative}. We observe that
        most eigenvectors in real-world datasets are sign canonizable. For instance,
        on the ogbg-molhiv dataset, the ratio of non-canonizable eigenvectors is only
        2.8\%. A thorough discussion on the feasibilty of Assumption~\ref{ass:mild}
        is in Appendix~\ref{app:assumption1}. For the left non-canonizable eigenvectors,
        we could apply previous methods (like RandSign, SAN, and SignNet) to further
        eliminate their ambiguity, which could save more compute since there are only
        a few non-canonizable eigenvectors. In practice, we also obtain good
        performance by simply ignoring these non-canonizable features.
    \end{remark}

    \subsubsection{Basis Canonization with MAP}\label{sec:basis}

    We further extend MAP-sign to solve the more challenging basis ambiguity with
    multiple eigenvectors, named MAP-basis. Now, MAP relies on two conditions to
    produce canonical eigenvectors. The first one requires there are enough
    distinctive summary vectors to determine $d$ canonical eigenvectors.
    \begin{assumption}\label{ass:k}
        The number of distinctive angles $k$ (\textit{i.e.}, the number of summary
        vectors $\{\vx_i\}$) is larger or equal to the multiplicity $d$, \textit{i.e.},
        $k\geq d$.
    \end{assumption}

    Under this condition, we can determine each $\vu_i$ iteratively with the
    corresponding summary vector $\vx_i$. At the $i$-th step where $\vu_1,\dots,
    \vu_{i-1}$ have already been determined, we choose $\vu_i$ to be the vector
    that is closest to the summary vector $\vx_i$ in the orthogonal complement space
    of $\langle\vu_1,\dots,\vu_{i-1}\rangle$ in $V$:
    \begin{equation}\label{eq:basis}
        \begin{aligned}
            &\vu_i=\argmax_\vu\vu^\top\vx_i,\\
            \text{\textit{s.t.} }&\vu\in\langle\vu_1,\dots,\vu_{i-1}\rangle^
            \perp,|\vu|=1.
        \end{aligned}
    \end{equation}

    With a compact feasibility region, the maximum is attainable, and we can further
    show that the solution $\vu_i$ is unique (\textit{c.f.}, the proof of
    Theorem~\ref{thm:basis}). \Eqref{eq:basis} can be directly solved using the
    projection matrix of $\langle\vu_1,\dots,\vu_{i-1}\rangle^\perp$ (see details in
    Algorithm~\ref{alg:use} in Appendix~\ref{app:pseudo-code}). Repeating this process
    gives us a canonical basis of $V$. We also require non-orthogonality condition at
    each step to obtain a valid eigenvector.
    \begin{assumption}\label{ass:perp}
        For any $1\leq i\leq d$, $\vx_i$ is not perpendicular to $\langle\vu_1,\dots,
        \vu_{i-1}\rangle^\perp$.
    \end{assumption}

    We summarize MAP-basis in Algorithm~\ref{alg:basis} in
    Appendix~\ref{app:pseudo-code}. In Theorem~\ref{thm:basis}, we prove under
    Assumption~\ref{ass:k} and Assumption~\ref{ass:perp}, MAP-basis gives a
    basis-invariant, permutation-equivariant and universal canonization of $\mU$.
    Though, in this scenario, MAP-basis cannot canonize all basis canonizable
    features, and whether such an algorithm exists remains an open problem for future
    research. 
    \begin{theorem}\label{thm:basis}
        Under Assumption~\ref{ass:k} and Assumption~\ref{ass:perp}, our MAP-basis
        algorithm gives a basis-invariant, permutation-equivariant and universal
        canonization of $\mU$.
    \end{theorem}

    \begin{remark}
        Assumption~\ref{ass:k} and Assumption~\ref{ass:perp} exclude some symmetries
        of the eigenspace, which we discuss in details in Appendix~\ref{app:assumption2}.
        For random orthonormal matrices, the possibility that either assumption is
        violated is equal to $0$, which we verify with random simulation in
        Appendix~\ref{app:verify-basis}. For real-world datasets, these assumptions are
        not restrictive either. For instance, on the large ogbg-molpcba dataset,
        the eigenvalues violating Assumption~\ref{ass:k} or Assumption~\ref{ass:perp}
        make up 0.87\% of all eigenvalues. More statistics and discussions are
        provided in Appendix~\ref{app:assumption2}.
    \end{remark}

    \subsubsection{Summary}

    We provide the complete pseudo-code of MAP to address both sign and basis
    invariance in Appendix~\ref{app:pseudo-code}, along with a detailed time complexity
    analysis. Overall, the extra complexity introduced by MAP is $\mathcal{O}(n^2\log
    n)$, which is better than eigendecomposition itself with $\mathcal{O}(n^3)$. It
    only needs to be computed once per dataset and can be easily incorporated with
    various GNN architectures.

    Combining Theorem~\ref{thm:se-universal}, Theorem~\ref{thm:sign} and
    Theorem~\ref{thm:basis} gives us the universality of first-order GNNs
    \citep{invariant-universal} such that it respects all graph symmetries under
    certain assumptions (Assumptions~\ref{ass:mild}, \ref{ass:k} \& \ref{ass:perp}).
    We show that the probability of violating these assumptions asymptotically
    converges to zero on random graphs (see Appendix~\ref{app:random-graphs}).
    We also count the ratio of violation in real-world datasets in
    Table~\ref{tab:violation}, and observe that the ratio of uncanonizable eigenvectors
    is less than 10\% on all datasets. Thus, MAP greatly eases the harm caused by sign
    and basis ambiguities.
    
    So far only GNNs using higher-order tensors have achieved universality while
    respecting \emph{all} graph symmetries
    \citep{invariant-universal,universal-invariant-equivariant-gnn}, but they are
    typically computationally prohibitive in practice. It is still an open problem
    whether it is also possible for first-order GNNs. The proposed Laplacian
    Canonization presents a new approach in this direction trying to alleviate the
    harm of sign and basis ambiguities in a minimalist approach. By establishing
    universality-invariance results for first-order GNNs under certain assumptions,
    LC could hopefully could bring some insights to the GNN community.

    \section{Experiments}

    We evaluate the proposed MAP positional encoding on sparse MP-GNNs and
    Transformer GNNs using PyTorch \citep{pytorch} and DGL \citep{dgl}. For sparse
    MP-GNNs we consider GatedGCN \citep{gatedgcn} and PNA \citep{pna}, and for
    Transformer GNNs we consider SAN \citep{lpe} and GraphiT \citep{graphit}.
    We conduct experiments on three standard molecular benchmarks---ZINC \citep{zinc},
    OGBG-MOLTOX21 and OGBG-MOLPCBA \citep{ogb}. ZINC and MOLTOX21 are of medium
    scale with 12K and 7.8K graphs respectively, whereas MOLPCBA is of large scale
    with 437.9K graphs. Details about these datasets are described in
    Appendix~\ref{app:datasets}. We follow the same protocol as in \citet{signnet},
    where we replace their SignNet with a normal GNN and the Laplacian eigenvectors
    with our proposed MAP\@. We fairly compare several models on a fixed number of
    500K model parameters on ZINC and relax the model sizes to larger parameters for
    evaluation on the two OGB datasets, as being practised on their leaderboards
    \citep{ogb}. We also compare our results with GNNs using LapPE and random sign
    (RS) \citep{benchmarking-gnn}, GNNs using SignNet \citep{signnet}, and GNNs with
    no PE\@. For a fair comparison, all models use a limited number $k$ of
    eigenvectors for positional encodings. The results of all our experiments are
    presented in Table~\ref{tab:zinc}, \ref{tab:moltox21} \& \ref{tab:molpcba}.
    Further implementation details are included in Appendix~\ref{app:hyperparameters}.

    \subsection{Performance on Benchmark Datasets}

    As shown in Table~\ref{tab:zinc}, \ref{tab:moltox21} \& \ref{tab:molpcba}, using MAP
    improves the performance of all GNNs on all datasets, demonstrating that removing
    ambiguities of the eigenvectors is beneficial for graph-level tasks. First, by
    comparing models with LapPE and RS with models with no PE, we observe that the
    use of LapPE significantly improves the performance on ZINC, showing the
    benefits of incorporating expressive PEs with GNNs, especially with MP-GNNs
    whose expressive power is limited by the 1-WL test. However, on MOLTOX21 and
    MOLPCBA, using LapPE has no significant effects. This is because unlike ZINC,
    OGB-MOL* datasets contain additional structural features that are informative,
    \textit{e.g.}, if an atom is in ring, among others \citep{lspe}. Thus the
    performance gain by providing more positional information is less obvious.
    Second, MAP outperms LapPE with RS by a large margin especially on ZINC. Although
    RS also alleviates sign ambiguity by randomly flipping signs during training,
    MAP removes such ambiguity \emph{before} training, enabling the network to
    focus on the real meaningful features and achieves a better performance. Third,
    we also observe that MAP and SignNet achieve comparable performance. This is
    because both methods aim at the same goal---eliminating ambiguity. However,
    SignNet does so in the training stage while MAP does so in the pre-processing
    stage, thus the latter is more computationally efficient. Lastly, we would
    also like to highlight that as a kind of positional encoding, MAP can be
    easily incorporated with any GNN architecture by passing the
    \texttt{pre\_transform} function to the dataset class with a single line of
    code.

    \begin{table}[htbp]
        \centering
        \caption{Results on ZINC\@. All scores are averaged over 4 runs with 4 different seeds.}
        \begin{tabular}{ccccc}
            \toprule
            Model          & PE             & $k$            & \#Param        & MAE $\downarrow$ \\
            \midrule
            GatedGCN       & None           & 0              & 504K           & 0.251 ± 0.009 \\
            GatedGCN       & LapPE + RS     & 8              & 505K           & 0.202 ± 0.006 \\
            GatedGCN       & SignNet ($\phi(v)$ only) & 8              & 495K           & 0.148 ± 0.007 \\
            GatedGCN       & SignNet        & 8              & 495K           & 0.121 ± 0.005 \\
            GatedGCN       & MAP            & 8              & 486K           & \textbf{0.120 ± 0.002} \\
            \midrule
            PNA            & None           & 0              & 369K           & 0.141 ± 0.004 \\
            PNA            & LapPE + RS     & 8              & 474K           & 0.132 ± 0.010 \\
            PNA            & SignNet        & 8              & 476K           & 0.105 ± 0.007 \\
            PNA            & MAP            & 8              & 462K           & \textbf{0.101 ± 0.005} \\
            \midrule
            SAN            & None           & 0              & 501K           & 0.181 ± 0.004 \\
            SAN            & MAP            & 16             & 230K           & \textbf{0.170 ± 0.012} \\
            \midrule
            GraphiT        & None           & 0              & 501K           & 0.181 ± 0.006 \\
            GraphiT        & MAP            & 16             & 329K           & \textbf{0.160 ± 0.006} \\
            \bottomrule
        \end{tabular}
        \label{tab:zinc}
    \end{table}

    \begin{table}[htbp]
        \centering
        \caption{Results on MOLTOX21\@. All scores are averaged over 4 runs with 4 different seeds.}
        \begin{tabular}{ccccc}
            \toprule
            Model          & PE             & $k$            & \#Param        & ROCAUC $\uparrow$ \\
            \midrule
            GatedGCN       & None           & 0              & 1004K          & 0.772 ± 0.006 \\
            GatedGCN       & LapPE + RS     & 3              & 1004K          & 0.774 ± 0.007 \\
            GatedGCN       & MAP            & 3              & 1505K          & \textbf{0.784 ± 0.005} \\
            \midrule
            PNA            & None           & 0              & 5245K          & 0.755 ± 0.008 \\
            PNA            & MAP            & 16             & 1951K          & \textbf{0.761 ± 0.002} \\
            \midrule
            SAN            & None           & 0              & 958K           & 0.744 ± 0.007 \\
            SAN            & MAP            & 12             & 1152K          & \textbf{0.766 ± 0.007} \\
            \midrule
            GraphiT        & None           & 0              & 958K           & 0.743 ± 0.003 \\
            GraphiT        & MAP            & 16             & 590K           & \textbf{0.769 ± 0.011} \\
            \bottomrule
        \end{tabular}
        \label{tab:moltox21}
    \end{table}

    \begin{table}[htbp]
        \centering
        \caption{Results on MOLPCBA\@. All scores are averaged over 4 runs with 4 different seeds.}
        \begin{tabular}{ccccc}
            \toprule
            Model          & PE             & $k$            & \#Param        & AP $\uparrow$ \\
            \midrule
            GatedGCN       & None           & 0              & 1008K          & 0.262 ± 0.001 \\
            GatedGCN       & LapPE + RS     & 3              & 1009K          & 0.266 ± 0.002 \\
            GatedGCN       & MAP            & 3              & 2658K          & \textbf{0.268 ± 0.002} \\
            \midrule
            PNA            & None           & 0              & 6551K          & 0.279 ± 0.003 \\
            PNA            & MAP            & 16             & 4612K          & \textbf{0.281 ± 0.003} \\
            \bottomrule
        \end{tabular}
        \label{tab:molpcba}
        \vspace{-0.2in}
    \end{table}

    \subsection{Empirical Understandings}

    \textbf{Computation time.} We demonstrate the efficiency of MAP by measuring the
    pre-processing time and training time on the large OGBG-MOLPCBA dataset, and
    compare them with SignNet. For a fair comparison, we use identical model size,
    hyperparameters and random seed and conduct experiments on the same NVIDIA 3090
    GPU\@. The results are shown in Table~\ref{tab:time}. We observe that model with
    MAP train 41\% faster than its SignNet counterpart, saving 44 hours of training
    time. Since SignNet takes the form $\rho\bigl(\phi(\vu)+\phi(-\vu)\bigr)$ while
    we use models taking the form $\rho\bigl(\phi(\vu)\bigr)$, models with MAP would
    always train faster than those with SignNet under the same hyperparameters. We
    also observe that the pre-processing time is negligible compared with training
    time (< 3\%), since pre-processing only needs to be done once. This makes MAP
    overall more efficient while achieving the same goal of tackling ambiguities.

    \begin{table}[htbp]
        \centering
        \caption{Comparison of pre-processing and training time between models with
        MAP or SignNet as PE, on the MOLPCBA dataset. Experiments are run with the
        same model size and hyperparameters, the same random seed, on the same NVIDIA
        3090 GPU\@.}
        \begin{tabular}{cccc}
            \toprule
            Model              & Pre-processing time & Training Time &  Total Time \\
            \midrule
            GatedGCN + MAP     & 1.70\,h             & 63.02\,h    & 64.72\,h \\
            GatedGCN + SignNet & 0.27\,h             & 108.51\,h   & 108.78\,h \\
            \bottomrule
        \end{tabular}
        \label{tab:time}
    \end{table}

    \textbf{Spectral embedding dimension.} Next, we study the effects of $k$. We
    train GatedGCN with MAP on ZINC with different number of eigenvectors used in
    the PE and report the results in Table~\ref{tab:k}. The hyperparameters are the
    same across these experiments. It can be observed that the performance drops when
    $k$ is too small, meaning that SE provides crucial structural information
    to the model. Using larger $k$ has limited influence on the performance, meaning
    that the model relies more on low-frequency information of spectral embedding.

    \begin{table}[htbp]
        \centering
        \caption{Effects of $k$, where $k$ is the number of eigenvectors used.}
        \begin{tabular}{ccccc}
            \toprule
            $k$            & 0              & 4              & 8              & 16 \\
            \midrule
            Test MAE       & 0.256 ± 0.012  & 0.138 ± 0.004  & \textbf{0.120 ± 0.002}  & 0.124 ± 0.002 \\
            \bottomrule
        \end{tabular}
        \label{tab:k}
    \end{table}

    \textbf{Ablation study.} Finally, we conduct ablation study of MAP by removing
    each component of MAP and evaluate the performance. The results are shown in
    Table~\ref{tab:ablation}. Removing sign invariance hurts the performance the most,
    because most eigenvectors are single and thus sign ambiguity has the most
    influence on the model's performance. Removing basis invariance also has negative
    effects since a small portion of eigenvalues are multiple. Removing eigenvalues
    has moderate negative effects, showing that the incorporation of eigenvalue
    information is beneficial for the network.

    \begin{table}[htbp]
        \centering
        \caption{Effects of the three components of MAP (GatedGCN on ZINC).}
        \begin{tabular}{ccccc}
            \toprule
            PE             & full MAP       & without can.\ sign & without can.\ basis & without eigenvalues \\
            \midrule
            Test MAE       & \textbf{0.120 ± 0.002} & 0.131 ± 0.003    & 0.122 ± 0.003    & 0.125 ± 0.001 \\
            \bottomrule
        \end{tabular}
        \label{tab:ablation}
    \end{table}

    \section{Conclusion}

    In this paper, we explored a new approach called Laplacian Canonization for
    addressing sign and basis ambiguities of Laplacian eigenvectors while also
    preserving permutation invariance and universal expressive power. We developed
    a novel theoretical framework to characterize canonization and the canonizability
    of eigenvectors. Then we proposed a practical canonization algorithm called
    Maximal Axis Projection (MAP). Theoretically, it is sign/basis-invariant,
    permutation-equivariant and universal. Empirically, it canonizes most eigenvectors
    on synthetic and real-world data while showing promising performance on
    various datasets and GNN architectures.

\section*{Acknowledgements}
Yisen Wang was supported by National Key R\&D Program of China (2022ZD0160304), National Natural Science Foundation of China (62006153, 62376010, 92370129), Open Research Projects of Zhejiang Lab (No. 2022RC0AB05), and Beijing Nova Program (20230484344).

    \bibliographystyle{plainnat}
    \bibliography{references.bib}

    \newpage

    \appendix

    \renewcommand\thepart{}
    \renewcommand \partname{}
    \part{Appendix}
    \setcounter{secnumdepth}{4}
    \setcounter{tocdepth}{4}
    \parttoc
    
    \newpage

    \section{Why sign and basis invariance matters}\label{app:why}

    It is perfectly understandable that one may not see the significance and benefits
    of sign and basis invariance at once. For instance, below is a quote from one of
    the reviews of \citet{signnet} when they first submitted their paper to NeurIPS
    2022:
    \begin{quotation}
        It is not clear ``why'' to preserve the two symmetries this paper proposes
        to preserve including sign invariance and basis invariance. In graph \&
        molecule setting, two important symmetries are permutation (\textit{i.e.}\ if
        we permute the nodes/atoms, the output stays the same) and rotation
        (\textit{i.e.}\ if we rotate the 3D coordinates associating with each atom,
        the prediction for a physical property stays the same). I don't see a strong
        evidence that suggests preserving sign \& basis invariances on top of
        permutation \& rotation invariances brings any benefits.
    \end{quotation}
    Indeed, it is not easy to see why issues related with something matters if they
    don't even use it. One may overlook the importance of sign and basis invariance
    for two reasons: (1) people not familiar with Laplacian PE may not even realize
    sign/basis ambiguity is a thing; (2) people that does have some knowledge about
    Laplacian PE may not realize sign and basis ambiguities are markedly hindering
    its performance. Nonetheless, Laplacian PE does have advantages over other
    commonly used PEs that it is worth the efforts addressing their drawbacks.
    For these reasons we feel it is of great significance that we explain the
    motivation behind preserving sign and basis invariances in more details.

    \textbf{Why use spectral embedding: trade-off between universality and
    ambiguities.} Sign and basis ambiguities arise only in spectral embeddings, so
    the first important question is why to use them. There are many kinds of
    positional encodings in the literature, such as Random Walk PE (RWPE),
    Position-aware Encoding, Relational Pooling (RP), random features, \textit{etc}.
    Without addressing sign and basis ambiguities, some works even reported
    superior performance of RWPE over spectral embedding \citep{lspe,gps}. However,
    these other PEs all have drawbacks either in their expressive power or
    permutation equivariance. RP and random features are universal, but they do
    not guarantee permutation equivariance, which is an important inductive bias
    in graph-structured data. RWPE and positional-aware encoding are permutation
    equivariant, but their expressive power is limited. Thus one may wonder whether
    there exists a kind of positional encoding method that achieves both. Spectral
    Embedding, as it turns out, is both permutation equivariant and universal
    (when coupled with eigenvalue information, see Sec~\ref{sec:universal}). The
    good thing about SE is that through eigendecomposition, it transforms a
    ``harder'' permutation equivariance formulated by
    \begin{equation}\label{eq:harder}
        f(\mP^\top\hat{\mA}\mP)=\mP^\top f(\hat{\mA})\mP,
    \end{equation}
    into a ``easier'' permutation equivariance formulated by
    \begin{equation}\label{eq:easier}
        f(\mP\mU)=\mP f(\mU),
    \end{equation}
    where $\mP$ is an arbitrary permutation matrix, $\hat{\mA}$ is the normalized
    adjacency matrix of the graph and $\mU$ are its eigenvectors. Universality
    and equivariant results for \eqref{eq:harder} are only known for high-order
    tensor networks \citep{ign}, but similar results for \eqref{eq:easier} have
    long been known to be achievable even by first-order networks such as DeepSets
    \citep{deep-sets}. Thus SE enables efficient and universal graph neural
    networks that respects permutation equivariance. However, SE also brings new
    problems, sign and basis ambiguities. It is not fair to say models with SE
    are universal without achieving sign and basis invariance, since we expect graph
    neural networks to respect all symmetries among isomorphic graphs.

    \textbf{Sign and basis invariances are well-established problems in the
    literature.} Sign and basis ambiguities have been recognized by numerous works
    in the literature to be challenging and important issues when incorporating SE
    as positional encoding. Below we list some existing discussions in these papers.

    Quote from \citet{benchmarking-gnn}:
    \begin{quotation}
        We propose an alternative to reduce the sampling space, and therefore the
        amount of ambiguities to be resolved by the network. Laplacian eigenvectors
        are hybrid positional and structural encodings, as they are invariant by
        node re-parametrization. However, they are also limited by natural symmetries
        such as the arbitrary sign of eigenvectors (after being normalized to have
        unit length). The number of possible sign flips is $2^k$, where $k$ is the
        number of eigenvectors. In practice we choose $k\ll n$, and therefore $2^k$
        is much smaller than $n!$ (the number of possible ordering of the nodes).
        During the training, the eigenvectors will be uniformly sampled at random
        between the $2^k$ possibilities. If we do not seek to learn the invariance
        \textit{w.r.t.}\ all possible sign flips of eigenvectors, then we can remove
        the sign ambiguity of eigenvectors by taking the absolute value. This choice
        seriously degrades the expressivity power of the positional features.
    \end{quotation}

    Quote from \citet{lpe}:
    \begin{quotation}
        As noted earlier, there is a sign ambiguity with the eigenvectors. With the
        sign of $\phi$ being independent of its normalization, we are left with a
        total of $2^k$ possible combination of signs when choosing $k$ eigenvectors
        of a graph. Previous work has proposed to do data augmentation by randomly
        flipping the sign of the eigenvectors \citep{dgn,gt,benchmarking-gnn}, and
        although it can work when $k$ is small, it becomes intractable for large $k$.
    \end{quotation}

    Quote from \citet{dgn}:
    \begin{quotation}
        For instance, a pair of eigenvalues can have a multiplicity of 2 meaning that
        they can be generated by different pairs of orthogonal eigenvectors. For an
        eigenvalue of multiplicity 1, there are always two unit norm eigenvectors of
        opposite sign, which poses a problem during the directional aggregation.
        We can make a choice of sign and later take the absolute value. An alternative
        is to take a sample of orthonormal basis of the eigenspace and use each choice
        to augment the training.
    \end{quotation}

    Quote from \citet{graphit}:
    \begin{quotation}
        Note that eigenvectors of the Laplacian computed on different graphs could
        not be compared to each other in principle, and are also only defined up to
        a $\pm 1$ factor. While this raises a conceptual issue for using them in
        an absolute positional encoding scheme, it is shown in \citet{gt}---and
        confirmed in our experiments---that the issue is mitigated by the Fourier
        interpretation, and that the coordinates used in LapPE are effective in
        practice for discriminating between nodes in the same way as the position
        encoding proposed in \citet{transformer} for sequences. Yet, because the
        eigenvectors are defined up to a $\pm 1$ factor, the sign of the encodings
        needs to be randomly flipped during the training of the network.
    \end{quotation}

    Quote from \citet{gt}:
    \begin{quotation}
        In particular, \citet{benchmarking-gnn} make the use of available graph
        structure to pre-compute Laplacian eigenvectors
        \citep{laplacian-eigenmaps-for-reduction} and use them as node positional
        information. Since Laplacian PEs are generalization of the PE used in the
        original transformers \citep{transformer} to graphs and these better help
        encode distance-aware information (\textit{i.e.}, nearby nodes have similar
        positional features and farther nodes have dissimilar positional features),
        we use Laplacian eigenvectors as PE in Graph Transformer. Although these
        eigenvectors have multiplicity occuring due to the arbitrary sign of
        eigenvectors, we randomly flip the sign of the eigenvectors during training,
        following \citet{benchmarking-gnn}.
    \end{quotation}

    Quote from \citet{lspe}:
    \begin{quotation}
        Another PE candidate for graphs can be Laplacian Eigenvectors
        \citep{benchmarking-gnn,gt} as they form a meaningful local coordinate system,
        while preserving the global graph structure. However, there exists sign
        ambiguity in such PE as eigenvectors are defined up to $\pm 1$, leading to $2^k$
        number of possible sign values when selecting $k$ eigenvectors which a
        network needs to learn. Similarly, the eigenvectors may be unstable due to
        eigenvalue multiplicities.
    \end{quotation}

    Quote from \citet{signnet}:
    \begin{quotation}
        However, there are nontrivial symmetries that should be accounted for when
        processing eigenvectors. If $v$ is a unit-norm eigenvector, then so is $-v$,
        with the same eigenvalue. More generally, if an eigenvalue has higher
        multiplicity, then there are infinitely many unit-norm eigenvectors that can
        be chosen. Indeed, a full set of orthonormal eigenvectors is only defined up
        to a change of basis in each eigenspace. In the case of sign invariance, for
        any $k$ eigenvectors there are $2^k$ possible choices of sign. Accordingly,
        prior works randomly flip eigenvector signs during training in order to
        approximately learn sign invariance \citep{lpe,benchmarking-gnn}. However,
        learning all $2^k$ invariances is challenging and limits the effectiveness
        of Laplacian eigenvectors for encoding positional information. Sign
        invariance is a special case of basis invariance when all eigenvalues are
        distinct, but higher dimensional basis invariance is even more difficult to
        deal with, and we show that these higher dimensional eigenspaces are abundant
        in real datasets.
    \end{quotation}

    Quote from \citet{equivariant-and-stable}:
    \begin{quotation}
        \citet{embedding-equivalence} states that PE using the eigenvectors of the
        randomly permuted graph Laplacian matrix keeps permutation equivariant.
        \citet{benchmarking-gnn,lpe} argue that such eigenvectors are unique up to
        their signs and thus propose PE that randomly perturbs the signs of those
        eigenvectors. Unfortunately, these methods may have risks. They cannot provide
        permutation equivariant GNNs when the matrix has multiple eigenvalues, which
        thus are dangerous when applying to many practical networks. For example,
        large social networks, when not connected, have multiple 0 eigenvalues;
        small molecule networks often have non-trivial automorphism that may give
        multiple eigenvalues. Even if the eigenvalues are distinct, these methods are
        unstable. We prove that the sensitivity of node representations to the graph
        perturbation depends on the inverse of the smallest gap between two
        consecutive eigenvalues, which could be actually large when two eigenvalues
        are close.
    \end{quotation}

    \textbf{Sign and basis invariances make the learning tasks easier.} Permutation
    invariance is an important inductive bias in graphs in the sense that applying
    permutations on graphs results in isomorphic graphs and they should have the
    same properties and produce the same outputs. Empirically constraining the
    network to be invariant to such permutations benefits the performance, since
    otherwise the network has to learn these $n!$ ambiguities by itself. Similarly,
    sign and basis ambiguities result in $2^k$ and infinite possible choices of
    positional features respectively that a network has to learn, which could
    degrade its performance greatly. By preserving sign and basis invariance, a
    network will not need to learn the equivalence between all these possible
    choices and the learning task is significantly easier.

    \textbf{There are less ambiguity in signs than in permutations.} For a given
    graph, there are $n!$ possible ordering of the nodes that all represent the
    same graph, resulting in $n!$ number of ambiguities. Considering that the
    vast majority of eigenvectors on real-world data have distinct eigenvalue
    (Table~\ref{tab:m}), when selecting $k$ eigenvectors as positional encoding
    to present structural information, the number of ambiguities is $2^k$, which
    is much smaller than $n!$. This transformation from permutation ambiguity
    to sign ambiguity greatly reduces the amount of random noise in the input data
    and makes the model much more stable.

    \textbf{When sign invariance meets permutation equivariance.} Sign ambiguity
    is not an issue when we do not take permutation symmetry into account. For
    each eigenvector pair $\pm\vu$, we can choose the one whose first non-zero
    entry is positive. This simple \emph{canonization} solves the sign ambiguity
    problem with ease. The problem of sign ambiguity is only \emph{non-trivial}
    when combined with permutation equivariance, that is, we require the canonization
    process to be both sign-invariant and permutation-equivariant. Since the
    eigenvectors can now be permuted arbitrarily, there is no ``first'' non-zero
    entry and the simple canonization above fails.

    Moreover, such canonization only makes sense when we require it to be
    \emph{universal}. For instance, mapping all eigenvectors to $\vj=(1,1,\dots,1)$
    is also a sign-invariant and permutation-equivariant canonization, but this
    canonization provides us with no information. Thus one may wonder whether a
    network can be both sign-invariant, permutation-equivariant and universal at
    the same time. Unfortunately, as far as we know, no existing work addresses all
    of them. The universality of SignNet only considers sign invariance, as stated
    in their theorem:
    \begin{theorem}[Universal representation for SignNet]
        A continuous function $f\colon(\mathbb{S}^{n-1})^k\to\R^s$ is sign invariant,
        \textit{i.e.}\ $f(s_1v_1,\dots,s_kv_k)=f(v_1,\dots,v_k)$ for any $s_i\in\{-1,
        1\}$, if and only if there exists a continuous $\phi\colon\R^n\to\R^{2n-2}$
        and a continuous $\rho\colon\R^{(2n-2)k}\to\R^s$ such that
        \[
            f(v_1,\dots,v_k)=\rho\bigl([\phi(v_i)+\phi(-v_i)]_{i=1}^k\bigr).
        \]
    \end{theorem}
    In fact, being sign-invariant, permutation-equivariant and universal at the
    same time is not even \emph{well-defined}. As mentioned in the main text,
    for an uncanonizable eigenvector $\vu$ with $\vu=-\mP\vu$ for some permutation
    matrix $\mP$, let $f$ be an arbitrary continuous sign-invariant,
    permutation-equivariant and universal function and denote $f(\vu)=\vy$.
    The sign invariance property requires that $f(-\vu)=\vy$, while the permutation
    equivariance property requires that $f(\mP\vu)=\mP\vy$. Since $-\vu=\mP\vu$,
    we have $\vy=\mP\vy$. However $\vu\neq\mP\vu$, thus universality is not
    satisfied, leading to a contradiction. This can be illustrated in
    \Figref{fig:dilemma}, where the three properties cannot be met at the same
    time.
    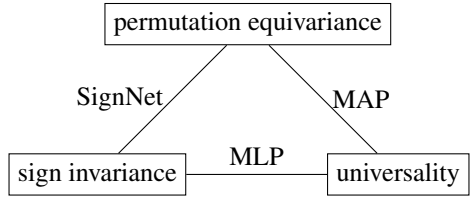
\begin{figure}[hbtp]
        \centering
        \begin{tikzpicture}
            \node [draw] at (0,0) (1) {sign invariance};
            \node [draw] at (2,2) (2) {permutation equivariance};
            \node [draw] at (4,0) (3) {universality};
            \draw [-] (1) -- node[pos=0.5, left]{SignNet} (2);
            \draw [-] (1) -- node[pos=0.5, above]{MLP} (3);
            \draw [-] (2) -- node[pos=0.5, right]{MAP} (3);
        \end{tikzpicture}
        \caption{The dilemma where permutation-equivariance, sign-invariance and
        universal expressive power cannot be achieved at the same time.}
        \label{fig:dilemma}
    \end{figure}

    \begin{remark}
        The \emph{universality} in Definition~\ref{def:canonization} is defined
        on eigenvectors $\mU\in\R^{n\times n}$. The permutation equivariance
        property of eigenvectors is defined as $f(\mP\mU)=\mP f(\mU)$, and here
        universality implies that if two eigenvectors are not equal up to permutation,
        then $f$ should be able to tell them apart. As mentioned above, such
        \emph{set-universal} networks (like DeepSets or MLP) cannot be permutation
        equivariant and sign invariant at the same time. In the context of graphs,
        we may also care about universality defined on the adjacency matrix $\mA\in
        \R^{n\times n}$, in which case the permutation equivariance property is
        defined as $f(\mP\mA\mP^\top)=\mP f(\mA)$. Such \emph{graph-universal}
        networks are always able to tell two non-isomorphic adjacency matrices apart.
        We can view $f(\mA)$ as a function of $\mU$ (we omit the eigenvalues for
        simplicity), and we denote $g(\mU)=f(\mA)$, then to achieve the graph
        universality of $f$, $g$ does not have to be set-universal, thus may not
        suffer from the dilemma above.
    \end{remark}
    
    Much effort has been devoted to the structure of invariant and equivariant
    networks in the literature, but few has studied the case when invariance
    and equivariance combine (in this case you may have to sacrifice some invariance,
    or some equivariance, or some expressive power). Our work takes a step in this
    direction, providing some insights to this dilemma and characterize the
    necessary and sufficient conditions where these three properties conflict
    with each other.

    \section{Related work}\label{app:related}

    \textbf{Theoretical Expressivity}\quad It has been shown that
    neighborhood-aggregation based GNNs are not more powerful than the WL test
    \citep{gin,higher-order}. Since then, many attempts have been made to endow
    graph neural networks with expressive power beyond the 1-WL test, such as
    injecting random attributes \citep{rp,random-features,p-gnn,embedding-equivalence,
    benchmarking-gnn}, injecting positional encodings \citep{distance-encoding,
    identity-aware}, and designing higher-order GNNs \citep{higher-order,
    ign,invariant-universal,ppgn,isomorphism-equivalence}.

    \textbf{Positional Encodings}\quad Recently, there are many works that aim to
    improve the expressive power of GNNs by positional encoding (PE). \citet{rp}
    assigned each node an identifier depending on its index. \citet{distance-encoding}
    proposed to use distances between nodes characterized by powers of the random
    walk matrix. \citet{p-gnn} proposed to use distances of nodes to a sampled anchor
    set of nodes. Using random node features can also improve the expressiveness of
    GNNs, even making them universal \citep{random-features,rni}, but it has several
    defects: (1) loss of permutation invariance, (2) slower convergence, (3) poor
    generalization on unseen graphs \citep{p-gnn,depth-vs-width}. More detailed
    discussion can be found in Appendix~\ref{app:pe}.

    \textbf{Laplacian PE}\quad Another PE candidate for graphs is Laplacian
    Eigenvectors \citep{gt,benchmarking-gnn} which form a meaningful local
    coordinate system while preserving global structure. However, there exists
    ambiguities in sign and basis that the network needs to learn, which harms the
    performance of GNNs. To address these ambiguities, many approaches have been
    proposed such as randomly flip the signs \citep{benchmarking-gnn}, use
    eigenfunctions over the edges \citep{lpe} or use invariant network architectures
    \citep{signnet}.

    \textbf{Sign Invariance}\quad Several prior works have proposed ways
    of addressing sign ambiguity. A popular heuristic is Random Sign (RS), which
    randomly flips the signs of eigenvectors during training to encourage the
    insensitivity to different signs \citep{benchmarking-gnn}. However, the network
    still has to learn these signs, leading to slower convergence and harder curve
    fitting task. \citet{resolving-sign} developed a data-dependent method to choose
    signs for each singular vector of a singular value decomposition. Still, in the
    worst case the signs chosen will be arbitrary, and they do not handle rotational
    ambiguities in higher dimensional eigenspaces. \citet{lpe} proposed to use the
    relative Laplace embeddings of two nodes. However, their approach suffers from a
    major computational bottleneck with $\mathcal{O}(n^4)$ complexity. \citet{signnet}
    proposed SignNet that passes both $\vu$ and $-\vu$ to the same network, adds the
    outputs together and then passes them to another network. As both outputs of
    $\vu$ and $-\vu$ need to be computed, their approach adds to the training cost.
    Our approach happens at pre-processing time with $\mathcal{O}(n^3)$ complexity.

    \textbf{Basis Invariance}\quad The only works we know addressing basis ambiguity
    are \citet{signnet} and \citet{spe}. \citet{signnet} proposed BasisNet that is
    basis invariant. However, to achieve universality, the components of BasisNet need to
    use higher order tensors in $\R^{n^k}$ where $k$ can be as large as $\frac{n(n-1)}2$
    \citep{invariant-universal}, rendering BasisNet impractical. \citet{spe}
    generalized BasisNet and proposed SPE that is both basis-invariant and stable
    to perturbation to the Laplacian matrix, but it still suffers from exponential
    complexity as BasisNet. In this paper, we showed that under certain assumptions,
    the issue of basis ambiguity can be resolved more efficiently in the
    pre-processing stage.

    There is also a literature beyond GNN that also considers spectral embeddings of
    graphs and going around the sign/basis ambiguity problems.
    \citet{point-cloud-registration} proposed to address sign/basis ambiguities
    using optimal transport theory that involves solving a non-convex optimization
    problem, thus it could be less efficient than our approach. \citet{eld} proposed
    to symmetrize the embedding using a heuristic measure called ELD that is quite
    similar to the form of SignNet, while our MAP algorithm offers an axis projection
    approach and establish its theoretical guarantees.

    \textbf{Canonical Forms}\quad The theory of canonical forms has been widely
    studied in mathematics \citep{general-theory} and applied to many fields of
    machine learning. \citet{condor} proposed a self-supervised method named ConDor
    that learns to canonicalize the 3D orientation and position for full and partial
    3D point clouds. \citet{canonical-fields} presented Canonical Field Network
    (CaFi-Net), a self-supervised method to canonicalize the 3D pose of instances
    from an object category represented as neural ﬁelds, specifically neural radiance
    ﬁelds (NeRFs). \citet{canonical-capsules} proposed an unsupervised capsule
    architecture for 3D point clouds. \citet{equivariance-canonicalization} proposed
    to decouple the equivariance and prediction components of neural networks by
    predicting a transformation that can be used to align the input data to some
    canonical pose. After this alignment the remaining layers do not have to be
    equivariant anymore. However, as far as we know, no existing work addresses the
    issue of \emph{ambiguities}, which arises when some function (\textit{e.g.},
    eigendecomposition) is required to be both \emph{equivariant} on the input
    space and \emph{invariant} on the output space. No existing work addresses the
    \emph{canonizability} of input data, which only arises when the invariance and
    equivariance property conflicts. Instead, we develop a novel theoretical
    framework that addresses these issues.

    \section{Implementation and verification of MAP}

    \subsection{Complete pseudo-code of MAP}\label{app:pseudo-code}

    A simplified pseudo-code for sign canonization with MAP is shown in
    Algorithm~\ref{alg:sign}. A simplified pseudo-code for basis canonization with
    MAP is shown in Algorithm~\ref{alg:basis}. ($\vj=(1,1,\dots,1)\in\R^n$)

    \begin{algorithm}[htbp]
        \caption{Maximal Axis Projection for eliminating sign ambiguity}
        \begin{algorithmic}
            \Require Input graph $\gG=(\sV,\sE,\mX)$
            \Ensure Spectral embedding of $\gG$
            \State Calculate the eigendecomposition $\hat{\mA}=\mU\mLambda\,\mU^\top$
            \Comment{$\mathcal{O}(n^3)$ complexity}
            \For{each eigenvector $\vu$}\Comment{$\mathcal{O}(n^2\log n)$ complexity}
                \State $\vx_i\gets\sum_{\ve_j\in\mathcal{B}_i}\ve_j+c\vj,\ i=1,\dots,k$
                \State $\vx_h\gets$ non-orthogonal summary vector with smallest $h$ (\eqref{eq:find-non-orthogonal})
                \If{$\vx_h$ is not \texttt{<none>}}
                \State $\vu\gets\vu^*$ (choose the direction with $\vu^\top\vx_h>0$ as in \eqref{eq:sign-canonization})
                \Else
                \State $\vu\gets\vu$ (no canonization)
                \EndIf
            \EndFor
            \State \Return{$\mU$}
        \end{algorithmic}
        \label{alg:sign}
    \end{algorithm}

    \begin{algorithm}[htbp]
        \caption{Maximal Axis Projection for eliminating basis ambiguity}
        \begin{algorithmic}
            \Require Eigenvalue $\lambda$ with multiplicity $d>1$
            \Ensure Spectral embedding corresponding to $\lambda$
            \State Calculate the eigenvectors $\mU\in\R^{n\times d}$ of $\lambda$ through eigendecomposition
            \Comment{$\mathcal{O}(n^3)$ complexity}
            \For{$i=1,2,\dots,d$}\Comment{$\mathcal{O}(n^2d)$ complexity}
                \State $\vx_i\gets\sum_{\ve_j\in\mathcal{B}_i}\ve_j+c\vj$
                \State Choose $\vu_i\in\langle\vu_1,\dots,\vu_{i-1}\rangle^\perp$, $|\vu_i|=1$, \textit{s.t.}\ $f(\vu_i)=\vu_i^\top\vx_i$ is maximized (\eqref{eq:basis})
            \EndFor
            \State \Return{$\mU_0\coloneqq[\vu_1,\dots,\vu_d]$}
        \end{algorithmic}
        \label{alg:basis}
    \end{algorithm}

    The complete pseudo-code of MAP is shown in Algorithm~\ref{alg:use}. Despite
    the sophisticated workflow, programmers do not have to know the principles
    of our algorithm to use it. The entire module can be passed as a
    \texttt{pre\_transform} function to the dataset class with a single line of code.

    \begin{algorithm}[htbp]
        \caption{Maximal Axis Projection}
        \begin{algorithmic}
            \Require Graph $\gG=(\sV,\sE)$, its normalized adjacency matrix $\hat{\mA}$
            \Ensure Maximal Axis Projection of $\gG$
            \State Calculate the eigendecomposition $\hat{\mA}=\mU\mLambda\,\mU^
            \top$\Comment{$\mathcal{O}(n^3)$ complexity}
            \For{each single eigenvector $\vu\in\R^n$}\Comment{$\mathcal{O}(n^2\log n)$
            complexity}
            \State $\mathit{proj}\gets(|\vu\vu^\top\ve_1|,|\vu\vu^\top\ve_2|,
            \dots,|\vu\vu^\top\ve_n|)$\Comment{$\mathcal{O}(n)$ complexity}
            \State $\mathit{len},\mathit{ind}\gets$ \Call{Sort}{$\mathit{proj}$}
            \Comment{$\mathcal{O}(n\log n)$ complexity}
            \State $\mathit{len}\gets$ \Call{Unique}{$\mathit{len}$}
            \State $k\gets|\mathit{len}|$
            \For{$i=1,2,\dots,k$}\Comment{$\mathcal{O}(n)$ complexity}
            \State $\vx_i\gets\sum_j\ve_{\mathit{ind}[j]}$ such that $\mathit{proj}_
            {\mathit{ind}[j]}=\mathit{len}[i]$
            \State $\vx_i\gets\vx_i+c\vj$
            \State $\vu_0\gets$ \Call{Normalize}{$\vu\vu^\top\vx_i$}
            \If{$\lVert\vu_0\rVert>\varepsilon$}\Comment{floating-point errors are
            considered}
            \State substitute $\vu$ with $\vu_0$
            \State \textbf{break}
            \EndIf
            \EndFor
            \EndFor
            \For{each multiple eigenvalue and its eigenvectors $\mU\in\R^{n\times d}$}
            \Comment{$\mathcal{O}(n^2dm)$ complexity}
            \State $\mathit{proj}\gets(|\mU\mU^\top\ve_1|,|\mU\mU^\top\ve_2|,
            \dots,|\mU\mU^\top\ve_n|)$\Comment{$\mathcal{O}(n^2d)$ complexity}
            \State $\mathit{len},\mathit{ind}\gets$ \Call{Sort}{$\mathit{proj}$}
            \State $\mathit{len}\gets$ \Call{Unique}{$\mathit{len}$}
            \If{$k<d$}
            \State \textbf{break}\Comment{Assumption~\ref{ass:k} not satisfied}
            \EndIf
            \For{$i=1,2,\dots,d$}\Comment{$\mathcal{O}(n)$ complexity}
            \State $\vx_i\gets\sum_j\ve_{\mathit{ind}[j]}$ such that $\mathit{proj}_
            {\mathit{ind}[j]}=\mathit{len}[i]$
            \State $\vx_i\gets\vx_i+c\vj$
            \EndFor
            \State $\mU_\mathrm{span}\gets$ empty matrix of shape $n\times 0$
            \State $\mU_\mathrm{perp}\gets\mU$\Comment{orthogonal complementary space}
            \For{$i=1,2,\dots,d$}\Comment{$\mathcal{O}(nd^4)$ complexity}
            \State $\vu_i\gets\mU_\mathrm{perp}\mU_\mathrm{perp}^\top\vx_i$
            \Comment{$\mathcal{O}(nd)$ complexity}
            \If{$\lVert\vu_i\rVert<\varepsilon$}\Comment{floating-point errors are
            considered}
            \State \textbf{break}\Comment{Assumption~\ref{ass:perp} not satisfied}
            \EndIf
            \State $\vu_i\gets$ \Call{Normalize}{$\vu_i$}
            \State substitute $\mU_{:,i}$ with $\vu_i$
            \State $\mU_\mathrm{span}\gets[\mU_\mathrm{span},\vu_i]$
            \State $\mU_\mathrm{base}\gets\mU_\mathrm{span}$
            \For{$j=1,2,\dots,d$}\Comment{$\mathcal{O}(nd^3)$ complexity}
            \State $\mU_\mathrm{temp}\gets[\mU_\mathrm{base},\mU_{:,j}]$
            \If{$\rank(\mU_\mathrm{temp})=j+1$}\Comment{$\mathcal{O}(nd^2)$ complexity}
            \State $\mU_\mathrm{base}\gets\mU_\mathrm{temp}$
            \EndIf
            \If{$\mU_\mathrm{base}\in\R^{n\times d}$}
            \State \textbf{break}
            \EndIf
            \EndFor
            \State $\mU_\mathrm{base}\gets$ \Call{GramSchmidtOrthogonalization}{$\mU_
            \mathrm{base}$}\Comment{$\mathcal{O}(nd^2)$ complexity}
            \State $\mU_\mathrm{perp}\gets(\mU_\mathrm{base})_{:,i+1:d}$
            \EndFor
            \EndFor
            \State $\mU\gets\mU\mLambda^\frac12$\Comment{$\mathcal{O}(n^2)$ complexity}
            \State \Return{$\mU$}
        \end{algorithmic}
        \label{alg:use}
    \end{algorithm}

    There are three steps in Algorithm~\ref{alg:use}. The first is to eliminate sign
    ambiguity with $\mathcal{O}(n^2\log n)$ time complexity. The second is to eliminate
    basis ambiguity with $\mathcal{O}(n^2dm)$ time complexity, where $d$ is the
    multiplicity of the eigenvalue and $m$ is the number of multiple eigenvalues
    of $\hat{\mA}$. The third is to incorporate eigenvalue information with
    $\mathcal{O}(n^2)$ time complexity. The overall time complexity is $\mathcal{O}
    (n^3)$ with the bottleneck being the eigendecomposition.

    The second part of our algorithm (eliminating basis ambiguity) has time complexity
    $\mathcal{O}(n^2dm)$. We point out that in real-world datasets $m$ is often
    quite small. As shown in Table~\ref{tab:m}, multiple eigenvalues only make up
    around $5\,\%$ of all eigenvalues, thus in practice the time complexity of
    the second part is $\mathcal{O}(n^2d)$, better than eigendecomposition.
    (Note it is important to take floating-point errors into account when counting
    these multiple eigenvalues)

    \begin{table}[htbp]
        \centering
        \caption{The number of multiple eigenvalues in real-world datasets.}
        \resizebox{\textwidth}{!}{
            \begin{tabular}{cccccccc}
                \toprule
                Dataset       & ogbg-molesol  & ogbg-molfreesolv & ogbg-molhiv   & ogbg-mollipo  & ogbg-moltox21 & ogbg-moltoxcast & ogbg-molpcba \\
                \midrule
                \#multiple eigenvalues & 738           & 286           & 52367         & 5391          & 8772          & 10556         & 491247 \\
                \#all eigenvalues & 13420         & 4654          & 952055        & 104669        & 129730        & 141042        & 10627757 \\
                Ratio & 5.50\,\%     & 6.15\,\%     & 5.50\,\%     & 5.15\,\%     & 6.76\,\%     & 7.48\,\%     & 4.62\,\% \\
                \bottomrule
            \end{tabular}
        }
        \label{tab:m}
    \end{table}

    Some parts of Algorithm~\ref{alg:use} have time complexity $\mathcal{O}(nd^4)$.
    We also point out that $d$ is usually small in real datasets. We show the number
    of eigenvalues and their multiplicities in logarithmic scale in Figure~\ref{fig:multiplicity}.

    \begin{figure}[htbp]
        \centering
        \includegraphics[scale=0.59]{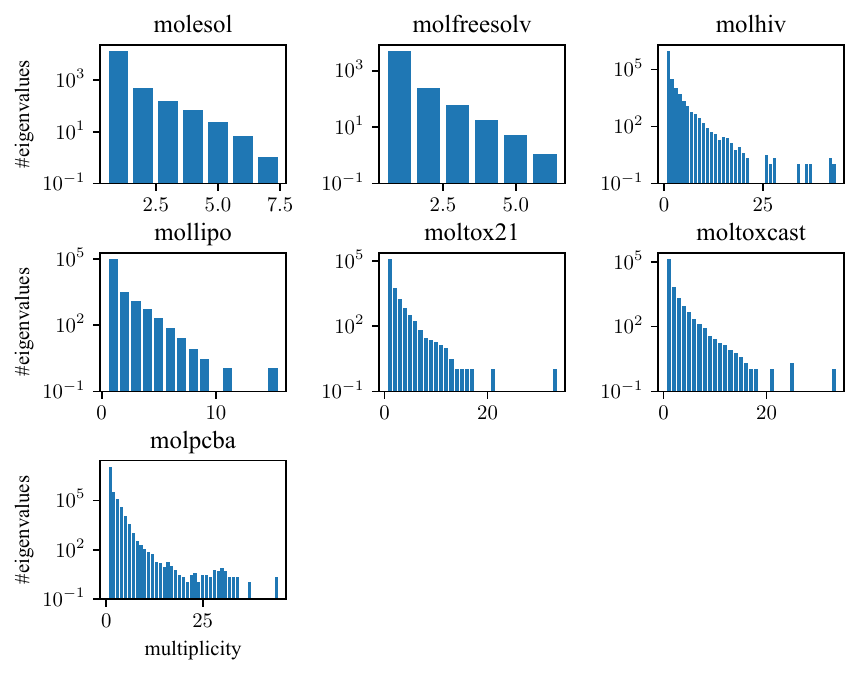}
        \caption{\#Eigenvalues \textit{w.r.t.}\ their multiplicities in
        real-world datasets (in logarithmic scale).}
        \label{fig:multiplicity}
    \end{figure}

    The details of the datasets are listed in Table~\ref{tab:datasets}.

    \begin{table}[htbp]
        \centering
        \caption{The details of the datasets. Reg: Regression; Bin: Binary classification.}
        \resizebox{\textwidth}{!}{
            \begin{tabular}{cccccccc}
                \toprule
                Dataset       & ogbg-molesol  & ogbg-molfreesolv & ogbg-molhiv   & ogbg-mollipo  & ogbg-moltox21 & ogbg-moltoxcast & ogbg-molpcba \\
                \midrule
                \#Graphs      & 1128          & 642           & 41127         & 4200          & 7831          & 8576          & 437929 \\
                Avg \#Nodes   & 13.3          & 8.7           & 25.5          & 27.0          & 18.6          & 18.8          & 26.0 \\
                Avg \#Edges   & 13.7          & 8.4           & 27.5          & 29.5          & 19.3          & 19.3          & 28.1 \\
                \#Tasks       & 1             & 1             & 1             & 1             & 12            & 617           & 128 \\
                Task Type     & Reg           & Reg           & Bin           & Reg           & Bin           & Bin           & Bin \\
                Metric        & RMSE          & RMSE          & ROC-AUC       & RMSE          & ROC-AUC       & ROC-AUC       & AP \\
                \bottomrule
            \end{tabular}
        }
        \label{tab:datasets}
    \end{table}

    \subsection{Verifying the correctness of Theorem~\ref{thm:sign}}
    \label{app:verify-sign}

    We verify the correctness of Theorem~\ref{thm:sign} through random simulation.
    The program is shown in Algorithm~\ref{alg:verify-sign}. Let $\mU$ be a
    random orthonormal matrix, $\mP$ be a random permutation matrix and $\mS$ be a
    random sign matrix (diagonal matrix of $1$ and $-1$). We pass $\mU$, $\mP\mU$,
    $\mU\!\mS$, $\mP\mU\!\mS$ to the \textsc{UniqueSign} function (Algorithm~\ref{alg:sign})
    and compare the outputs. If our algorithm is correct, $\mU$ and $\mU\!\mS$
    should have invariant outputs, while $\mU$, $\mP\mU$ and $\mP\mU\!\mS$ should
    have equivariant outputs.

    \begin{algorithm}[htbp]
        \caption{Verify the correctness of Theorem~\ref{thm:sign}}
        \begin{algorithmic}
            \State $\mathit{p\_correct}\gets0$, $\mathit{q\_correct}\gets0$, $\mathit
            {pq\_correct}\gets0$, $\mathit{total}\gets0$
            \For{$i=1,2,\dots,\mathit{trials}$}
            \State $n\gets$ a random positive integer
            \State $\mU\gets$ a random orthonormal matrix in $\R^{n\times n}$
            \State $\mU_0\gets$ \Call{UniqueSign}{$\mU$}
            \State $\mP\gets$ a random permutation matrix
            \State $\mV\gets\mP\mU$
            \State $\mV_0\gets$ \Call{UniqueSign}{$\mV$}
            \State $\mathit{p\_correct}\gets\mathit{p\_correct}+1$ \textbf{if}
            $|\mP\mU_0-\mV_0|<\varepsilon$\Comment{test permutation-equivariance}
            \State $\mS\gets$ a random sign matrix (diagonal matrix of $1$ and $-1$)
            \State $\mW\gets\mU\!\mS$
            \State $\mW_0\gets$ \Call{UniqueSign}{$\mW$}
            \State $\mathit{q\_correct}\gets\mathit{q\_correct}+1$ \textbf{if}
            $|\mU_0-\mW_0|<\varepsilon$\Comment{test uniqueness}
            \State $\mY\gets\mP\mW$
            \State $\mY_0\gets$ \Call{UniqueSign}{$\mY$}
            \State $\mathit{pq\_correct}\gets\mathit{pq\_correct}+1$ \textbf{if}
            $|\mP\mU_0-\mY_0|<\varepsilon$\Comment{test both}
            \State $\mathit{total}\gets\mathit{total}+1$
            \EndFor
            \State print out the values of $\mathit{p\_correct}$, $\mathit{q\_correct}$,
            $\mathit{pq\_correct}$ and $\mathit{total}$
        \end{algorithmic}
        \label{alg:verify-sign}
    \end{algorithm}

    We conduct 1000 trials. The results are $\mathit{p\_correct}=\mathit{q\_correct}=
    \mathit{pq\_correct}=\mathit{total}=1000$, showing that Algorithm~\ref{alg:sign}
    is both permutation-equivariant and unique (unambiguous).

    \subsection{Verifying the correctness of Theorem~\ref{thm:basis}}
    \label{app:verify-basis}

    We verify the correctness of Theorem~\ref{thm:basis} through random simulation.
    The program is shown in Algorithm~\ref{alg:verify-basis}. Let $\mU$ be a random
    orthonormal matrix in $\R^{n\times d}$, $\mP$ be a random permutation matrix and
    $\mQ$ be a random orthonormal matrix in $\R^{d\times d}$. We pass $\mU$, $\mP\mU$,
    $\mU\!\mQ$, $\mP\mU\!\mQ$ to the \textsc{UniqueBasis} function (Algorithm~\ref{alg:basis})
    and compare the outputs. If our algorithm is correct, $\mU$ and $\mU\!\mQ$
    should have invariant outputs, while $\mU$, $\mP\mU$ and $\mP\mU\!\mQ$ should
    have equivariant outputs.

    \begin{algorithm}[htbp]
        \caption{Verify the correctness of Theorem~\ref{thm:basis}}
        \begin{algorithmic}
            \State $\mathit{p\_correct}\gets0$, $\mathit{q\_correct}\gets0$, $\mathit
            {pq\_correct}\gets0$, $\mathit{total}\gets0$
            \For{$i=1,2,\dots,\mathit{trials}$}
            \State $n\gets$ a random positive integer (greater than $1$)
            \State $d\gets$ a random positive integer (less than $n$)
            \State $\mU\gets$ a random orthonormal matrix in $\R^{n\times d}$
            \State $\mU_0\gets$ \Call{UniqueBasis}{$\mU$}
            \State $\mP\gets$ a random permutation matrix
            \State $\mV\gets\mP\mU$
            \State $\mV_0\gets$ \Call{UniqueBasis}{$\mV$}
            \State $\mathit{p\_correct}\gets\mathit{p\_correct}+1$ \textbf{if}
            $|\mP\mU_0-\mV_0|<\varepsilon$\Comment{test permutation-equivariance}
            \State $\mQ\gets$ a random orthonormal matrix in $\R^{d\times d}$
            \State $\mW\gets\mU\!\mQ$
            \State $\mW_0\gets$ \Call{UniqueBasis}{$\mW$}
            \State $\mathit{q\_correct}\gets\mathit{q\_correct}+1$ \textbf{if}
            $|\mU_0-\mW_0|<\varepsilon$\Comment{test uniqueness}
            \State $\mY\gets\mP\mW$
            \State $\mY_0\gets$ \Call{UniqueBasis}{$\mY$}
            \State $\mathit{pq\_correct}\gets\mathit{pq\_correct}+1$ \textbf{if}
            $|\mP\mU_0-\mY_0|<\varepsilon$\Comment{test both}
            \State $\mathit{total}\gets\mathit{total}+1$
            \EndFor
            \State print out the values of $\mathit{p\_correct}$, $\mathit{q\_correct}$,
            $\mathit{pq\_correct}$ and $\mathit{total}$
        \end{algorithmic}
        \label{alg:verify-basis}
    \end{algorithm}

    We conduct 1000 trials. The results are $\mathit{p\_correct}=\mathit{q\_correct}=
    \mathit{pq\_correct}=\mathit{total}=1000$, showing that Algorithm~\ref{alg:basis}
    is both permutation-equivariant and unique. The function \textsc{UniqueBasis}
    raises an assertion error when either Assumption~\ref{ass:k} or
    Assumption~\ref{ass:perp} is violated, so Algorithm~\ref{alg:verify-basis} also
    shows that random orthonormal matrices violate these assumptions with probability
    $0$.

    \subsection{Verifying the correctness of Theorem~\ref{thm:se-universal}}
    \label{app:verify-universal}

    We conduct experiment on the \textsc{Exp} dataset proposed in \citet{rni}, which
    is designed to explicitly evaluate the expressive power of GNNs. The dataset
    consists of a set of 1-WL indistinguishable non-isomorphic graph pairs, and each
    graph instance is a graph encoding of a propositional formula. The classification
    task is to determine whether the formula is satisfiable (SAT). Since the graph
    pairs in the \textsc{Exp} dataset are not distinguishable by 1-WL test, if a model
    shows above 50\% accuracy on this dataset, it should have expressive power
    beyond the 1-WL algorithm.

    The models we used on \textsc{Exp} dataset are as follows: an 8-layer GCN, GIN
    \citep{gin}, PPGN \citep{ppgn}, 1-2-3-GCN-L \citep{higher-order}, 3-GCN \citep{rni},
    DeepSets-RNI (DeepSets with Random Node Initialization (RNI) \citep{rni}), GCN-RNI
    (GCN with Random Node Initialization (RNI)), Linear-RSE (linear network with
    RSE), and DeepSets-RSE. GCN and GIN belongs to the family of MP-GNNs whose expressive
    power is bounded by 1-WL, and RNI is a method to improve the expressive
    power of MP-GNNs. We verify the expressive power gain of RSE-based models by
    comparing them with GCN, and evaluate their efficiency by comparing them with
    other expressive models.

    We evaluate all models on the \textsc{Exp} dataset using 10-fold cross-validation,
    and train each model for 500 epochs per fold. Mean test accuracy across all
    folds is measured and reported. The results are reported in Table~\ref{tab:exp}.
    In addition, we also measure the learning curves of models to show their
    convergence rate, as shown in \Figref{fig:exp-converge}.
    
    \begin{table}[htbp]
        \centering
        \begin{minipage}{.49\linewidth}
            \centering
            \captionof{table}{Accuracy results on \textsc{Exp}.}
            \begin{tabular}{lc}
                \toprule
                Model & Test Accuracy (\%) \\
                \midrule
                GCN   & $50.0$ \\
                GIN & $50.0$ \\
                PPGN & $50.0$ \\
                1-2-3-GCN-L & $50.0$ \\
                3-GCN & $99.7\pm 0.0$ \\
                DeepSets-RNI & $50.0$ \\
                GCN-RNI & $97.6\pm 2.5$ \\
                Linear-RSE & $99.1\pm 1.8$ \\
                DeepSets-RSE & $\mathbf{99.8\pm 0.5}$ \\
                \bottomrule
            \end{tabular}
            \label{tab:exp}
        \end{minipage}
        \begin{minipage}{.49\linewidth}
            \centering
            \includegraphics[width=\linewidth]{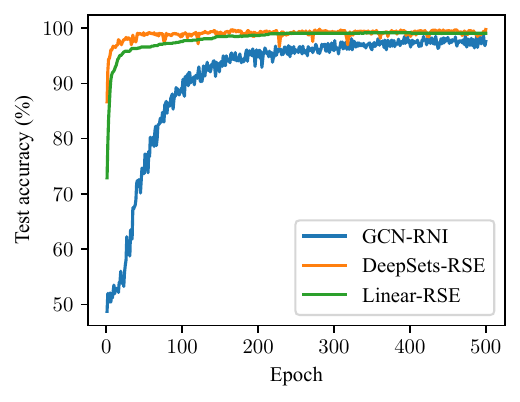}
            \captionof{figure}{Learning curves on \textsc{Exp}.}
            \label{fig:exp-converge}
        \end{minipage}
    \end{table}

    In Table~\ref{tab:exp}, we observe that vanilla GCN and DeepSets-RNI only achieves
    50\% accuracy, because they do not have expressive power beyond the 1-WL test.
    DeepSets-RSE achieves the best performance among all models with a near 100\% accuracy,
    which demonstrates the expressive power of RSE-based models is beyond the 1-WL
    test. Other models, namely Linear-RSE and GCN-RNI, also achieve comparable
    performance. It's worth mentioning that even a simple linear model (Linear-RSE)
    could achieve performance above 99\%. This indicates that the universal
    expressiveness of models with RSE is mostly from RSE itself, rather
    than the network structure.

    From \Figref{fig:exp-converge}, we find that the convergence rate of RSE-based
    models are much faster than their RNI-based counterpart. This is because
    RSE-based models are deterministic, while RNI-based models are random and require
    more training epochs to converge. The structures of DeepSets-RSE and Linear-RSE
    are simpler, which means they also train much faster than GCN-RNI\@.

    \section{Synthetic experiment on basis invariance}\label{basis-synthetic}

    As mentioned in Appendix~\ref{app:related}, SignNet \citep{signnet} only deals
    with sign ambiguity, while BasisNet \citep{signnet} has a prohibitive
    computational overhead. On the other hand, our proposed MAP addresses basis
    ambiguity efficiently, albeit with the existence of uncanonizable eigenspaces.
    In this section, we conduct a synthetic experiment to verify the ability of MAP
    on addressing basis ambiguity. We use graph isomorphic testing, a traditional
    graph task. Our focus is on 10 non-isomorphic random weighted graphs $\gG_1,
    \dots,\gG_{10}$, all exhibiting basis ambiguity issues (with the first three
    eigenvectors belonging to the same eigenspace). We sample 20 instances for each
    graph, introducing different permutations and basis choices for the initial
    eigenspace. The dataset is then split into a 9:1 ratio for training and testing,
    respectively. The task is a 10-way classification, where the aim is to
    determine the isomorphism of a given graph to one of the 10 original graphs. The
    model is given the first 3 eigenvectors as input (\textit{i.e.}\ $k=3$). The
    results are averaged over 4 different runs.

    \begin{table}[htbp]
        \centering
        \caption{Test accuracy of the synthetic graph isomorphic testing task with
        DeepSets \citep{deep-sets} using different PEs.}
        \begin{tabular}{lc}
            \toprule
            Positonal Encoding & \multicolumn{1}{c}{Accuracy} \\
            \midrule
            LapPE           & 0.11 ± 0.08 \\
            LapPE + RS      & 0.10 ± 0.09 \\
            LapPE + SignNet & 0.10 ± 0.03 \\
            LapPE + MAP     & \textbf{0.84 ± 0.21} \\
            \bottomrule
        \end{tabular}
        \label{tab:basis-synthetic}
    \end{table}

    As evident from the results in Table~\ref{tab:basis-synthetic}, approaches that
    address sign ambiguity (like RandSign and SignNet) cannot obtain nontrivial
    performance on this task. Conversely, MAP shows commendable performance. The 84\%
    accuracy, although impressive, indicates potential avenues for further enhancement.
    We believe this synthetic task could also serve as a valuable benchmark for future
    studies addressing basis invariance through canonization. Code for this
    experiment is available at
    \url{https://github.com/GeorgeMLP/basis-invariance-synthetic-experiment}.

    \section{Other positional encoding methods}\label{app:pe}

    In this paper, we proposed MAP, which is a kind of positional encodings. In
    the field of graph representation learning, many other positional encoding
    methods have also been proposed. \citet{rp} proposed Relational Pooling (RP)
    that assigns each node with an identifier that depends on the index ordering.
    They showed that RP-GNN is strictly more expressive than the original WL-GNN\@.
    However, to ensure permutation equivalence, we have to account for all possible
    $n!$ node orderings, which is computationally intractable. \citet{p-gnn} proposed
    learnable position-aware embeddings by computing the distance of a target node
    to an anchor-set of nodes, and showed that P-GNNs have more expressive power than
    existing GNNs. However, the expressive power of their model is dependent on the
    random selection of anchor sets. \citet{random-features,rni} proposed to use full
    or partial random node features and proved that their model has universal
    expressive power, but it has several defects: (1) the loss of permutation
    invariance, (2) slower convergence, and (3) poor generalization on unseen graphs
    \citep{p-gnn,depth-vs-width}. \citet{distance-encoding} proposed Distance Encoding
    (DE) that captures the distance between the node set whose representation is to
    be learned and each node in the graph. They proved that DE can distinguish node
    sets embedded in almost all regular graphs where traditional GNNs always fail.
    However, their approach fails on distance regular graphs, and computation of
    power matrices can be a limiting factor for their model's scalability. \citet{lspe}
    proposed Learnable Structural and Positional Encodings (LSPE) that decouples
    structural and positional representations by inserting MPGNNs-LSPE layers and
    showed promising performance on three molecular benchmarks.

    In particular, we will show that with Random Node Initialization (RNI),
    (1) A linear network is universal on a fixed graph,
    and (2) An MLP with just one additional message passing layer can be
    universal on arbitrary graphs.
    
    In our work, we denote RNI as concatenating a random matrix to the input node
    features. The random matrix can be sampled from Gaussian distribution, uniform
    distribution, etc. Without loss of generality, we will assume that each entry
    of the random matrix is sampled independently from the standard Gaussian
    distribution $N(0,1)$.

    \begin{definition}
        A GNN with RNI is defined by concatenating a random matrix\/ $\rmR$ to the
        input node features, i.e., $\mX'=[\mX,\rmR]$, where $\mX$ are the original
        node features, $\mX'$ are the modified node features and each entry of\/
        $\rmR$ is sampled independently from the standard Gaussian distribution
        $N(0,1)$. The value of\/ $\rmR$ is sampled at every forward pass of GNN\@.
    \end{definition}

    To study the effects of RNI on the expressiveness of GNNs, we consider two types
    of tasks: tasks on a fixed graph (e.g., node classification) and tasks on
    arbitrary graphs (e.g. graph classification). On a fixed graph, we aim to learn
    a function $f\colon\R^{n\times d}\to\R^{n\times d'}$ that transforms the feature
    of each node $v_i$ to a presentation vector $\mZ_{i,:}\in\R^{1\times d'}$. We
    claim that a linear GNN with RNI in the form
    \begin{equation}\label{eq:linear-rni}
        [\mX,\rmR]\mW=\mZ
    \end{equation}
    is universal, where $\mW\in\R^{d\times d'}$ are the network parameters, and $\mZ
    \in\R^{n\times d'}$ is the desired output. In other words, we have the
    universality theorem of linear GNNs with RNI on a fixed graph:

    \begin{theorem}\label{thm:linear-rni-universal}
        On a fixed graph $\gG$, a linear GNN with RNI defined by \Eqref{eq:linear-rni}
        is equivariant and can produce any prediction $\mZ\in\R^{n\times d'}$ with
        probability\/ $1$.
    \end{theorem}

    We prove Theorem~\ref{thm:linear-rni-universal} in Appendix~\ref{app:linear-rni-universal}.

    On arbitrary graphs, the target function is not only dependent on the node features,
    but on the graph structure $\hat{\mA}\in\R^{n\times n}$ as well. Let $\Omega
    \subset\R^{n\times d}\times\R^{n\times n}$ be a compact set of graphs with
    $[\mX,\hat{\mA}]\in\Omega$, where $\mX$ are the node features and $\hat{\mA}$
    is the normalized adjacency matrix.
    We wish to learn a function $f\colon\Omega\to\R$ that transforms each graph to its
    label. Since $f$ is also dependent on $\hat{\mA}$, an MLP-based network with
    $\mX'$ as input is not expressive enough, and we need additional graph convolutional
    layers to obtain information about the graph structure. However, \citet{rgnn}
    proved that with just \emph{one} additional message passing layer, an MLP with
    RNI can approximate any continuous invariant graph function $f$.

    \begin{theorem}[\citet{rgnn}]\label{thm:rgnn-universal}
        Given a compact set of graphs\/ $\Omega\subset\R^{n\times d}\times\R^{n\times n}$,
        a GNN with one message passing layer, an MLP network with RNI can approximate
        an arbitrary continuous invariant graph function $f\colon\Omega\to\R$ to an
        arbitrary precision.
    \end{theorem}

    Theorem~\ref{thm:rgnn-universal} is a direct inference of the proof of
    Proposition~1 in \citet{rgnn}, where the authors constructed a RGNN that first
    transfers the graph structural information to the node features through a message
    passing layer, and then approximates $f$ with a DeepSets network, which is an
    MLP-based network.

    \section{Toy Examples}\label{app:toy}

    \subsection{Toy examples for Algorithm~\ref{alg:sign}}

    \begin{figure}[htbp]
        \centering
        \def\svgwidth{130mm}
        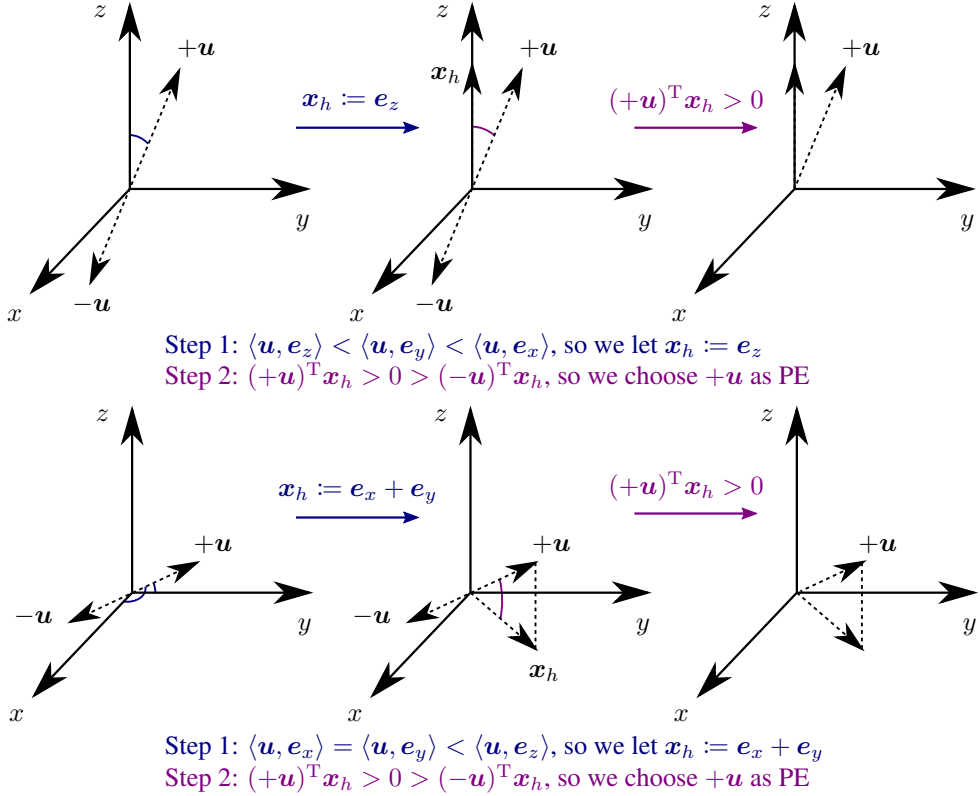
        \caption{A toy example illustrating our algorithm for eliminating sign
        ambiguity. \textbf{Top}: The angle between the $z$-axis and $\vu$ is the
        smallest, so we choose $+\vu$ to maximize $\vu^\top\ve_z$.
        \textbf{Bottom}: The angle between both $x$ and $y$-axes and $\vu$ are the
        smallest, so we choose $+\vu$ to maximize $\vu^\top(\ve_x+\ve_y)$.}
        \label{fig:sign-example}
    \end{figure}

    We give toy examples to help illustrate our MAP-sign
    algorithm. As shown on the top row of \Figref{fig:sign-example}, we have $n=3$
    and two possible sign choices for the eigenvector $\vu$. Our algorithm first
    compares the angles (or equivalently, the absolute value of inner product) between
    $\vu$ and the standard basis: $\ve_x,\ve_y,\ve_z$, and pick the smallest one
    (the one with the largest absolute inner product), in this case $\ve_z$. Thus
    we let $\vx_h=\ve_z$ and choose the sign that maximize $\vu^\top\vx_h$.
    In the first example we have $(+\vu)^\top\vx_h>0>(-\vu)^\top\vx_h$,
    thus $+\vu$ is chosen instead of $-\vu$. This choice is \emph{unique} and
    \emph{permutation-equivariant}.
    
    It is possible though, that the angle between $\vu$ and more than 1 basis vectors
    are equal. As shown on the bottom row of \Figref{fig:sign-example}, the angle
    between both $\ve_x,\ve_y$ and $\vu$ are maximum. In this case we let $\vx_h=
    \ve_x+\ve_y$ be their sum and maximize $\vu^\top\vx_h$, thus $+\vu$ is
    chosen. A special case is when $\vu$ and $\vx_h$ are perpendicular. If this
    happens, we go on to pick the basis vector with the second largest angle and
    continue this process.

    \subsection{Toy examples for Algorithm~\ref{alg:basis}}

    \begin{figure}[htbp]
        \centering
        \def\svgwidth{100mm}
        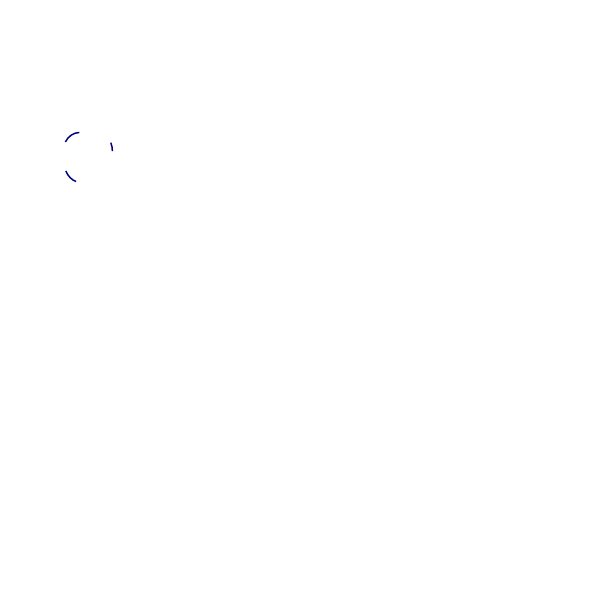
        \caption{A toy example illustrating our algorithm for eliminating basis
        ambiguity. First, we sort the angles between $V$ and the standard basis
        vectors, and set $\vx_1\coloneqq\ve_x,\vx_2\coloneqq\ve_y$. Next, we choose
        $\vu_1\in V$ that maximizes $\vu_1^\mathrm{T}\vx_1$. Finally, we choose
        $\vu_2\in\langle\vu_1\rangle^\perp$ that maximizes $\vu_2^\mathrm{T}\vx_2$.
        This gives us a unique basis $\vu_1,\vu_2$ of $V$.}
        \label{fig:basis-example}
    \end{figure}

    We give a toy example to help illustrate our MAP-basis algorithm. As shown in
    \Figref{fig:basis-example}, we have $n=3$, $d=2$ with a two-dimensional
    eigenspace $V$. Our algorithm first compares the angles (or equivalently,
    the length of orthogonal projection) between $V$ and the standard basis:
    $\ve_x,\ve_y,\ve_z$, and pick the two smallest one (the two with the largest
    and second largest lengths of orthogonal projection), in this case $\ve_x$ and
    $\ve_y$. Thus we let $\vx_1=\ve_x$, $\vx_2=\ve_y$. Then we choose $\vu_1\in V$
    that maximizes $\vu_1^\top\vx_1$, which is just the orthogonal projection of
    $\vx_1$ onto $V$. Finally we choose $\vu_2\in\langle\vu_1\rangle^\perp$ that
    maximizes $\vu_2^\top\vx_2$, which is the orthogonal projection of $\vx_2$ onto
    $\langle\vu_1\rangle^\perp$. This gives us a basis $\vu_1,\vu_2$ of $V$. This
    choice is \emph{unique} and \emph{permutation-equivariant}.

    It is possible though, that the angle between $V$ and more than 1 basis vectors
    are equal. For example, if $\langle V,\ve_x\rangle=\langle V,\ve_y\rangle$ are
    both the smallest angles between $V$ and the standard basis, then we let
    $\vx_1=\ve_x+\ve_y$. The same goes for the second smallest angle, the third
    smallest angle, and so on.

    \section{Discussions on the assumptions}

    \subsection{Discussions on Assumption~\ref{ass:mild}}\label{app:assumption1}

    The purpose of Section~\ref{sec:sign} is to uniquely determine the signs of
    the eigenvectors while also preserving their permutation-equivariance. One could
    easily come up with simple solutions such as choosing the signs such that the
    sum of the entries of eigenvectors are positive, or signs such that the element
    with the greatest absolute value is positive. In fact, \citet{signnet} has
    proposed a more general way of choosing signs:
    \begin{quotation}
        We also consider ablations in which we \ldots{} choose a canonical sign for each
        eigenvector by maximizing the norm of positive entries.
    \end{quotation}

    This ``canonical sign'' approach did not work well because (1) a large percentage
    of eigenvectors in real-world datasets have the same norm for positive and
    negative entries, thus this approach does not actually solve sign ambiguity
    of these eigenvectors; (2) this approach cannot be generalized to multiple
    eigenvalues. As a reference, we list the number of eigenvectors that ``canonical
    sign'' fails in Table~\ref{tab:canonical}.

    \begin{table}[htbp]
        \centering
        \caption{The number of eigenvectors with the same norm for positive and
        negative entries, the total number of eigenvectors, and the ratio of
        eigenvectors that ``canonical sign'' approach fails in real-world datasets.}
        \resizebox{\textwidth}{!}{
            \begin{tabular}{cccccccc}
                \toprule
                Dataset         & ogbg-molesol & ogbg-molfreesolv & ogbg-molhiv & ogbg-mollipo     & ogbg-moltox21    & ogbg-moltoxcast  & ogbg-molpcba \\
                \midrule
                \#Violation     & 4060         & 2135             & 209854      & 15115            & 31365            & 35174            & 1536415 \\
                \#Eigenvectors  & 14991        & 5600             & 1049163     & 113568           & 145459           & 161088           & 11373137 \\
                Ratio           & 27.08\,\%    & 38.13\,\%        & 20.00\,\%   & 9.84\,\%         & 21.56\,\%        & 21.84\,\%        & 13.51\,\% \\
                \bottomrule
            \end{tabular}
        }
        \label{tab:canonical}
    \end{table}

    Some examples of such eigenvectors in real-world datasets are as follows.
    \begin{center}
\begin{verbatim}
[ 0.0000,  0.0000, -0.0000,  0.0000,  0.1364, -0.0965, -0.0965, -0.1575,
  0.0965, -0.0000,  0.1575, -0.0965, -0.2784,  0.2227, -0.1364, -0.0000,
  0.3341, -0.2363, -0.2363, -0.1364,  0.2227, -0.0000, -0.0487,  0.0345,
  0.0345, -0.0877,  0.0620,  0.0620,  0.2784, -0.2227,  0.1364,  0.0000,
 -0.3341,  0.2363,  0.2363,  0.1364, -0.2227],
[ 0.0000,  0.0000,  0.0000, -0.0796,  0.0796,  0.0975, -0.0796, -0.0563,
  0.0000,  0.2087, -0.3615,  0.2951,  0.0000, -0.1815,  0.1815,  0.4767,
 -0.4767, -0.0029,  0.0051, -0.0042,  0.0000, -0.0458,  0.0458,  0.0416,
 -0.0416, -0.1495,  0.2589, -0.2114,  0.0000, -0.0975,  0.0975, -0.1139,
  0.1139],
[ 0.0000,  0.0000,  0.0000, -0.0025,  0.0025,  0.0031, -0.0025, -0.0018,
  0.0000, -0.0025,  0.0043, -0.0035, -0.0000, -0.0987,  0.0987,  0.0952,
 -0.0952, -0.1271,  0.2202, -0.1798,  0.0000, -0.3234,  0.3234,  0.1436,
 -0.1436,  0.1313, -0.2275,  0.1858,  0.0000, -0.2524,  0.2524,  0.4381,
 -0.4381].
\end{verbatim}
    \end{center}

    Assumption~\ref{ass:mild}, on the other hand, is less restrictive. It requires
    that the eigenvectors are not perpendicular to at least one of the vectors
    $\vx_i$. For random unit vectors or random weighted graphs,
    Assumption~\ref{ass:mild} has $0$ possibility of being violated. The number of
    eigenvectors violating Assumption~\ref{ass:mild} in real-world datasets are
    listed in Table~\ref{tab:ass1}. It can be observed that the ratio of violation
    tends to become smaller as the graph size becomes larger.

    \begin{table}[htbp]
        \centering
        \caption{The number of eigenvectors violating Assumption~\ref{ass:mild},
        the total number of eigenvectors, and the ratio of violation in real-world
        datasets. We ignore small graphs with no more than 5 nodes.}
        \resizebox{\textwidth}{!}{
            \begin{tabular}{cccccccc}
                \toprule
                Dataset          & ogbg-molesol     & ogbg-molfreesolv & ogbg-molhiv      & ogbg-mollipo     & ogbg-moltox21    & ogbg-moltoxcast  & ogbg-molpcba \\
                \midrule
                \#Violation      & 727              & 388              & 29558            & 3328             & 5418             & 6032             & 343088 \\
                \#Eigenvectors   & 14551            & 5048             & 1049101          & 113568           & 144421           & 159987           & 11372381 \\
                Ratio            & 5.00\,\%         & 7.69\,\%         & 2.82\,\%         & 2.93\,\%         & 3.75\,\%         & 3.77\,\%         & 3.02\,\% \\
                \bottomrule
            \end{tabular}
        }
        \label{tab:ass1}
    \end{table}

    \subsection{Discussions on Assumption~\ref{ass:k} and Assumption~\ref{ass:perp}}
    \label{app:assumption2}

    Assumption~\ref{ass:k} requires that the projections of $\ve_1,\ve_2,\dots,\ve_n$
    on $V$ have at least $d$ distinct lengths, while Assumption~\ref{ass:perp}
    requires that each $\vx_i$ is not perpendicular to $\langle\vu_1,\dots,\vu_{i-1}
    \rangle^\perp$, which is a subspace of $V$. We illustrate two examples of
    these assumptions being violated when $n=3$ and $d=2$ in Figure~\ref{fig:ass23}.
    In the left figure, the projections of $\ve_1,\ve_2,\ve_3$ on $V$ all have the
    same lengths, thus $k=1<d$, violating Assumption~\ref{ass:k}. In the right figure,
    $\vx_2$ is perpendicular to the orthogonal complementary space of $\mathop{\mathrm
    {span}}(\vu_1)$ in $V$, violating Assumption~\ref{ass:perp}. We observe that
    the eigenspace $V$ needs to obey certain kinds of symmetries in order to violate
    either assumptions. For random orthonormal matrices and random weighted graphs,
    these assumptions have $0$ possibility of being violated; and in real-world
    datasets, the ratio of violation tends to become smaller as the graph size
    becomes larger.

    \begin{figure}[htbp]
        \centering
        \def\svgwidth{100mm}
        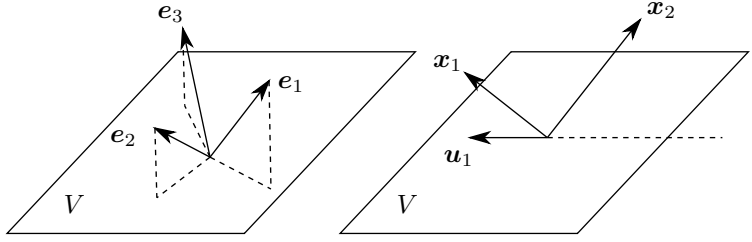
        \caption{Examples of Assumption~\ref{ass:k} and Assumption~\ref{ass:perp}
        being violated when $n=3$ and $d=2$. \textbf{Left:} the projections of
        $\ve_1,\ve_2,\ve_3$ are of the same length, thus $k=1<d$, violating
        Assumption~\ref{ass:k}. \textbf{Right:} $\vx_2$ is perpendicular to the
        orthogonal complementary space of $\mathop{\mathrm{span}}(\vu_1)$ in $V$,
        violating Assumption~\ref{ass:perp}.}
        \label{fig:ass23}
    \end{figure}

    The number of multiple eigenvalues violating these assumptions in real-world
    datasets are listed in Table~\ref{tab:ass23}. Indeed there is still around
    $20\,\%$ of violation on large datasets, but since multiple eigenvalues only
    make up a small portion of all eigenvalues (Table~\ref{tab:m}), the ratio of
    violation is relatively small and has negligible influence on the model
    performance.

    \begin{table}[htbp]
        \centering
        \caption{The number of eigenvalues violating Assumption~\ref{ass:k} or
        Assumption~\ref{ass:perp} in real-world datasets. $n_1$: the number of
        eigenvalues violating Assumption~\ref{ass:k}. $n_2$: the number of
        eigenvalues not violating Assumption~\ref{ass:k} but violating
        Assumption~\ref{ass:perp}. $N_1$: the total number of multiple eigenvalues.
        $p_1$: the ratio of multiple eigenvalues violating Assumption~\ref{ass:k}.
        $p_2$: the ratio of multiple eigenvalues violating Assumption~\ref{ass:perp}.
        $N_2$: the total number of eigenvalues. $p_3$: the ratio of all eigenvalues
        violating Assumption~\ref{ass:k}. $p_4$: the ratio of all eigenvalues
        violating Assumption~\ref{ass:perp}. We ignore small graphs with no more
        than 5 nodes.}
        \resizebox{\textwidth}{!}{
            \begin{tabular}{cccccccc}
                \toprule
                Dataset          & ogbg-molesol     & ogbg-molfreesolv & ogbg-molhiv      & ogbg-mollipo     & ogbg-moltox21    & ogbg-moltoxcast  & ogbg-molpcba \\
                \midrule
                $n_1$            & 39               & 30               & 5315             & 440              & 646              & 873              & 30844 \\
                $n_2$            & 126              & 45               & 6329             & 599              & 1073             & 1450             & 61318 \\
                $N_1$            & 738              & 286              & 52367            & 5391             & 8772             & 10556            & 491247 \\
                $p_1$            & 5.28\,\%         & 10.49\,\%        & 10.15\,\%        & 8.16\,\%         & 7.36\,\%         & 8.27\,\%         & 6.28\,\% \\
                $p_2$            & 17.07\,\%        & 15.73\,\%        & 12.09\,\%        & 11.11\,\%        & 12.23\,\%        & 13.74\,\%        & 12.48\,\% \\
                $N_2$            & 13420            & 4654             & 952055           & 104669           & 129730           & 141042           & 10627757 \\
                $p_3$            & 0.29\,\%         & 0.64\,\%         & 0.56\,\%         & 0.42\,\%         & 0.50\,\%         & 0.62\,\%         & 0.29\,\% \\
                $p_4$            & 0.94\,\%         & 0.97\,\%         & 0.66\,\%         & 0.57\,\%         & 0.83\,\%         & 1.03\,\%         & 0.58\,\% \\
                \bottomrule
            \end{tabular}
        }
        \label{tab:ass23}
    \end{table}

    \section{Discussions on random graphs}\label{app:random-graphs}

    In this section we summarize some existing results on random graphs and more
    generally random matrices. These discussions help us to have a better understanding
    of how eigenvalues and eigenvectors of random graphs distribute, and the
    probability that they are uncanonizable by our MAP algorithm. Due to the
    theoretical complexity, we do make some simplifications in our discussions.

    We first give some basic definitions about random matrices and random graphs.

    \begin{definition}[Wigner matrix]
        Let $\xi,\zeta$ be real random variables with mean zero. We say $W$ is
        a \textbf{Wigner matrix} of size $n$ with atom variables $\xi,\zeta$ if
        $W=(w_{ij})_{i,j=1}^n$ is a random real symmetric $n\times n$ matrix that
        satisfies the following conditions.
        \begin{itemize}
            \item (independence) $\{w_{ij}\colon 1\leq i\leq j\leq n\}$ is a
            collection of independent random variables.
            \item (off-diagonal entries) $\{w_{ij}\colon 1\leq i<j\leq n\}$ is a
            collection of independent and identically distributed (iid) copies of
            $\xi$.
            \item (diagonal entries) $\{w_{ii}\colon 1\leq i\leq n\}$ is a
            collection of iid copies of $\zeta$.
        \end{itemize}
    \end{definition}

    If $\xi$ and $\zeta$ have the same distribution, we say $W$ is a Wigner matrix
    with atom variable $\xi$. We always assume that $\xi$ is \emph{non-degenerate},
    namely that there is no value $c$ such that $\mathbb{P}(\xi=c)=1$.

    \begin{definition}[symmetric Bernoulli matrix]
        Let $0<p<1$, and take $\xi$ to be the random variable
        \[
            \xi\coloneqq
            \begin{cases}
                1-p,&\text{with probability }p,\\
                -p,&\text{with probability }1-p.
            \end{cases}
        \]
        Then $\xi$ has zero mean. Let $B_n(p)$ denote the $n\times n$ Wigner matrix
        with atom variable $\xi$. We refer to $B_n(p)$ as a \textbf{symmetric
        Bernoulli matrix} (with parameter $p$).
    \end{definition}

    \begin{definition}[sub-exponential]
        A random variable $\xi$ is called \textbf{sub-exponential} with exponent
        $\alpha>0$ if there exists a constant $\beta>0$ such that
        \[
            \mathbb{P}(|\xi|>t)\leq\beta\exp(-t^\alpha/\beta)\quad\text{for all }
            t>0.
        \]
    \end{definition}

    \begin{definition}[Erd\H{o}s-R\'enyi random graph]
        Let $G(n,p)$ denote the Erd\H{o}s-R\'enyi random graph on $n$ vertices
        with edge density $p$. That is, $G(n,p)$ is a simple graph on $n$
        vertices such that each edge $\{i,j\}$ is in $G(n,p)$ with probability
        $p$, independent of other edges.
    \end{definition}

    Let $A_n(p)$ be the zero-one adjacency matrix of $G(n,p)$. $A_n(p)$ is not
    a Wigner matrix since its entries do not have mean zero. Let $\tilde{G}(n,p)$
    denote the Erd\H{o}s-R\'enyi random graph with loops on $n$ vertices with
    edge density $p$. Let $\tilde{A}_n(p)$ denote the zero-one adjacency matrix
    of $\tilde{G}(n,p)$. Technically, $\tilde{A}_n(p)$ is not a Wigner random
    matrix because its entries do not have mean zero. However, we can view
    $\tilde{A}_n(p)$ as a low rank deterministic perturbation of a Wigner matrix.
    That is, we can write $\tilde{A}_n(p)$ as
    \[
        \tilde{A}_n(p)=p\mJ_n+B_n(p),
    \]
    where $\mJ_n$ is the all-ones matrix.
    
    \textbf{Simplification in our discussions.} The adjacency matrix of
    Erd\H{o}s-R\'enyi random graphs can be viewed as a rank-one perturbation of
    a zero-mean symmetric random matrix, which has no effect on the distribution
    of eigenvalues in the limit $n\to\infty$ \citep{low-rank-perturbation}. In
    the finite case, the difference between eigenvalues and eigenvectors of the
    perturbed and unperturbed random matrices can be bounded as well, as stated
    in the following theorems.
    \begin{theorem}[\citet{random-perturbation}]
        Let $E$ be an $n\times n$ Bernoulli random matrix, and let $A$ be an
        $n\times n$ matrix with rank $r$. Let $\sigma_1\geq\sigma_2\geq\dots\geq
        \sigma_{\min\{m,n\}}\geq 0$ be singular values of $A$, $v_1,v_2,\dots,
        v_{\min\{m,n\}}$ be corresponding singular vectors of $A$, $\sigma_1'\geq
        \dots\geq\sigma_{\min\{m,n\}}'\geq 0$ be singular values of $A+E$,
        $v_1',\dots,v_{\min\{m,n\}}$ be corresponding singular vectors of $A+E$.
        For every $\varepsilon>0$ there exists constants $C_0,\delta_0>0$
        (depending only on $\varepsilon$) such that if $\delta>\delta_0$ and
        $\sigma_1\geq\max\{n,\sqrt{n}\delta\}$, then, with probability at least
        $1-\varepsilon$,
        \[
            \sin\angle(v_1,v_1')\leq C\frac{\sqrt{r}}{\delta}.
        \]
    \end{theorem}
    \begin{theorem}[\citet{random-perturbation}]
        Let $E$ be an $n\times n$ Bernoulli random matrix, and let $A$ be an
        $n\times n$ matrix with rank $r$ satisfying $\sigma_1\geq n$. For every
        $\varepsilon>0$, there exists a constant $C_0>0$ (depending only on
        $\varepsilon$) such that, with probability at least $1-\varepsilon$,
        \[
            \sigma_1-C\leq\sigma_1'\leq\sigma_1+C\sqrt{r}.
        \]
    \end{theorem}
    In particular, when the rank $r$ is significantly smaller than $n$, the
    bounds in the above theorems are significantly better. Thus the rank-one
    perturbation has little or no effect on the spectral distribution of the
    adjacency matrix. In the following discussions we will ignore the rank-one
    perturbation and assume that the adjacency matrix of random graphs have
    zero mean.

    We will study two cases: (1) the adjacency matrix of the random graph is
    continuously distributed (random weighted graph); (2) the adjacency matrix
    of the random graph is discretely distributed (random unweighted graph).

    It is easy to see that for a Wigner matrix with atom variables $\xi,\zeta$,
    if the distribution of $\xi$ is continuous, then with probability $1$ it
    has simple spectrum (\textit{i.e.}, all eigenvalues have multiplicity one),
    thus no basis ambiguity exists. For its single eigenvectors, we also expect
    these eigenvectors to be continuously distributed on the unit sphere, thus
    the probability that they are uncanonizable is equal to $1$. In fact, we
    can even give a explicit formula for the distribution of some simple
    random matrices, as follows.
    \begin{definition}[GOE]
        The Gaussian orthogonal ensemble (GOE) is defined as a Wigner random
        matrix with atom variables $\xi,\zeta$, where $\xi$ is a standard normal
        random variable and $\zeta$ is a normal random variable with mean zero
        and variance $2$.
    \end{definition}
    \begin{theorem}[\citet{introduction-to-random-matrices}, Section~2.5.1]
        Let $M$ be a $n\times n$ real symmetric matrix drawn from the GOE, with
        eigenvalues $\lambda_1\leq\dots\leq\lambda_n$ and corresponding eigenvectors
        $v_1,\dots,v_n$. Then the eigenvectors $v_1,\dots,v_n$ are uniformly
        distributed on
        \[
            S_+^{n-1}\coloneqq\{x=(x_1,\dots,x_n)\in S^{n-1}\colon x_1>0\},
        \]
        and the eigenvalues $\lambda_1,\dots,\lambda_n$ have joint density
        \[
            p_n(\lambda_1,\dots,\lambda_n)\coloneqq
            \begin{cases}
                \frac1{Z_n}\prod_{1\leq i<j\leq n}|\lambda_i-\lambda_j|\prod_{i=1}
                ^n\mathrm{e}^{-\lambda_i^2/4},&\text{if }\lambda_1\leq\dots\leq
                \lambda_n,\\
                0,&\text{otherwise},
            \end{cases}
        \]
        where $Z_n$ is a normalization constant.
    \end{theorem}
    Thus the eigenvectors of such Gaussian random matrices are uniformly distributed
    on the unit sphere, meaning they are canonizable with probability $1$.

    The discrete case is trickier, because when $n$ is finite, the probability
    that the eigenvectors are uncanonizable is no longer $0$. We wish that
    this probability can be upper bounded and asymptotically converges to zero.
    The first fact is that for a large class of random matrices, perturbed or
    not, they have simple spectrum with probability $1-o(1)$.
    \begin{theorem}[\citet{simple-spectrum}]
        Consider real symmetric random matrices $M_n=(\xi_{ij})_{1\leq i,j\leq n}$,
        where the entries $\xi_{ij}$ for $i\leq j$ are jointly independent with
        $\xi_{ji}=\xi_{ij}$, the upper-triangular entries $\xi_{ij},i<j$ have
        non-trivial distribution $\xi$ for some fixed $\mu>0$ such that
        $\mathbb{P}(\xi=x)\leq 1-\mu$ for all $x\in\R$. The diagonal entries
        $\xi_{ii},1\leq i\leq n$ can have an arbitrary real distribution (and
        can be correlated with each other), but are required to be independent
        of the upper diagonal entries $\xi_{ij},1\leq i<j\leq n$. Then for every
        fixed $A>0$ and $n$ sufficiently large (depending on $A,\mu$), the
        spectrum of $M_n$ is simple with probability at least $1-n^{-A}$.
    \end{theorem}
    Thus with probability $1-o(1)$ no basis ambiguity exists.

    For sign ambiguity, we expect that the probability that the single eigenvectors
    are canonizable asymptotically converges to $1$ as $n\to\infty$. Denote
    $\vj=(1,1,\dots,1)\in\R^n$ the all-one vector, $\hat{\vj}=(\frac1{\sqrt{n}},
    \frac1{\sqrt{n}},\dots,\frac1{\sqrt{n}})\in\R^n$ the normalized all-one vector.
    One obvious fact is that for an eigenvector $\vu\in\R^n$, if it is
    non-canonizable, then $\vu\cdot\vj=0$. Thus it suffices to show that the
    probability of the inner product between $\vu$ and $\vj$ being zero converges
    to $0$ as $n\to\infty$. This can be derived from the following theorem.
    \begin{theorem}[\citet{universal-properties-of-eigenvectors}]
        Let $\xi,\zeta$ be random variables such that
        \begin{itemize}
            \item $\xi$ and $\zeta$ are sub-exponential random variables,
            \item $\E(\xi)=\E(\zeta)=\E(\xi^3)=0$,
            \item $\E(\xi^2)=1,\E(\xi^4)=3,\E(\zeta^2)=2$,
        \end{itemize}
        and assume $\xi$ is a symmetric random variable. For each $n\geq 1$, let
        $W_n$ be an $n\times n$ Wigner matrix with atom variables $\xi,\zeta$.
        Let $\{a_n\}$ be a sequence of unit vectors with $a_n\in S^{n-1}$ such
        that $\lim_{n\to\infty}\lVert a_n\rVert_{\ell^\infty}=0$, and let $\{i_n\}$
        be a sequence of indices with $i_n\in[n]$. Then
        \[
            \sqrt{n}v_{i_n}(W_n)\cdot a_n\to N(0,1)
        \]
        in distribution as $n\to\infty$.
    \end{theorem}
    Thus in the limit as $n\to\infty$, by taking $a_n=\hat{\vj}$, we have the inner
    product between the eigenvector and $\vj$ following a normal distribution.
    The probability that they have zero inner product is $0$, indicating these
    eigenvectors are almost always canonizable.

    \textbf{Related readings.} We refer the readers to
    \citet{introduction-to-random-matrices} for a introduction of random matrices,
    \citet{low-rank-perturbation} for a survey about the singular values and vectors
    of low rank perturbations of large random matrices, \citet{eigenvectors-survey}
    for a comprehensive survey about eigenvectors of random matrices.

    \section{Further weakening the assumptions}

    \subsection{Further weakening Assumption~\ref{ass:mild}}

    In Section~\ref{sec:canonical}, we mentioned that some eigenvectors are
    intrinsically \emph{uncanonizable}, meaning that it is impossible to
    canonize them based on themselves. However, in graph-level tasks, the input
    is a whole graph with $n$ eigenvectors, not a single eigenvector. When the input
    graph is permuted by a permutation $\sigma$, all the eigenvectors are also
    permuted in the same way. This raises the question of whether we can canonize
    one eigenvector based on other eigenvectors in the same graph.

    Thus we propose to further weaken Assumption~\ref{ass:mild} in the following
    steps. For a given input graph, we first divide its eigenvectors into two
    sets: $S_1$, containing all uncanonizable eigenvectors; and $S_2$, containing
    all canonizable eigenvectors. Suppose $|S_1|=d_1$, $|S_2|=d_2$, where $d_1+d_2=n$.
    Define the matrix of canonizable eigenvectors
    \[
        \mU_\mathrm{can}=[\vu_1,\dots,\vu_{d_2}],\text{ where }\vu_1,\dots,\vu_{d_2}
        \text{ are all the eigenvectors in }S_2.
    \]
    We fist canonize all the eigenvectors in $S_2$, using our MAP-sign algorithm.
    This gives us a matrix of canonized eigenvectors $\mU_\mathrm{can}^*\in\R^{n\times
    d_2}$. Using a hash function, we can compress the matrix $\mU_\mathrm{can}^*$
    into a vector $\vu_\mathrm{can}$ that preserves all the information of $S_2$:
    \[
        \vu_\mathrm{can}\in\R^n,\text{ where }(\evu_\mathrm{can})_i=\mathop{\mathrm
        {hash}}\bigl\{(\mU_\mathrm{can}^*)_{i,:}\bigr\},1\leq i\leq n.
    \]
    Then we can define the summary vectors $\vx_i$ of $\vu_\mathrm{can}$ in the same
    way as in Section~\ref{sec:pratical}, and use them to canonize the eigenvectors
    in $S_1$. Denote the projection matrix $\mP=\vu_\mathrm{can}\vu_\mathrm{can}^\top$
    and the projected angles $\alpha_i=|\mP\ve_i|,1\leq i\leq n$. Assume that there
    are $k$ distinct values in $\{\alpha_i,i=1,\dots,n\}$, according to which we
    can divide all basis vectors $\{\ve_i\}$ into $k$ disjoint groups $\mathcal{B}_i$
    (arranged in descending order of the distinct angles). Each $\mathcal{B}_i$
    represents an equivalent class of axis that $\vu_\mathrm{can}$ has the same
    projection on. Then we define a summary vector $\vx_i$ of $\vu_\mathrm{can}$
    for the axes in each group $\mathcal{B}_i$ as their total sum $\vx_i=\sum_{\ve_j
    \in\mathrm{B}_i}\ve_j+c$, where $c$ is a tunable constant.

    For each $\vu\in S_1$, we can try to canonize it using the MAP-sign algorithm
    as in Section~\ref{sec:sign}, with the summary vectors of $\vu$ replaced by
    the summary vectors of $\vu_\mathrm{can}$. We find the first non-orthongonal
    summary vector with $\vu$, denoted by $\vx_h$, and choose the sign that
    maximizes $\vu^\top\vx_h$. In this way, we are able to canonize some eigenvectors
    in $S_1$ that are originally uncanonizable, thus further weakening
    Assumption~\ref{ass:mild}. The complete workflow is shown in
    Algorithm~\ref{alg:further-sign}.

    \begin{algorithm}[htbp]
        \caption{A stronger algorithm for eliminating sign ambiguity}
        \begin{algorithmic}
            \Require Input graph $\gG=(\sV,\sE,\mX)$
            \Ensure Spectral embedding of $\gG$
            \State Calculate the eigendecomposition $\hat{\mA}=\mU\mLambda\,\mU^\top$
            \State $S_1\gets$ the set of all uncanonizable eigenvectors of $\gG$
            \State $S_2\gets$ the set of all canonizable eigenvectors of $\gG$
            \State Canonize all eigenvectors in $S_2$, using Algorithm~\ref{alg:sign}
            \State $\mU_\mathrm{can}^*\gets[\vu_1^*,\dots,\vu_{d_2}^*]$, where $\vu_1
            ^*,\dots,\vu_{d_2}^*$ are all the canonized eigenvectors, $d_2=|S_2|$
            \State $\vu_\mathrm{can}\gets\bigl(\mathop{\mathrm{hash}}\{(\mU_\mathrm{can}
            ^*)_{1,:}\},\dots,\mathop{\mathrm{hash}}\{(\mU_\mathrm{can}^*)_{n,:}\}
            \bigr)^\top$
            \State $\mP\gets\vu_\mathrm{can}\vu_\mathrm{can}^\top$
            \State $\alpha_i\gets|\mP\ve_i|,1\leq i\leq n$
            \State $k\gets$ the number of distinct values in $\{\alpha_i\}$
            \State Divide all basis vectors $\{\ve_i\}$ into $k$ disjoint groups
            $\mathcal{B}_i$ according to the values of $\{\alpha_i\}$
            \State $\vx_i\gets\sum_{\ve_j\in\mathcal{B}_i}\ve_j+c\vj,i=1,\dots,k$
            \For{each eigenvector $\vu\in S_1$}
                \State $\vx_h\gets$ non-orthogonal summary vector with smallest $h$
                \State $\vu\gets-\vu$ if $\vu^\top\vx_h<0$
            \EndFor
            \State \Return{the canonized eigenvectors}
        \end{algorithmic}
        \label{alg:further-sign}
    \end{algorithm}

    \subsection{Further weakening Assumption~\ref{ass:k} and Assumption~\ref{ass:perp}}

    When reading Algorithm~\ref{alg:basis}, some may find it not as strong as
    Algorithm~\ref{alg:sign}. In Algorithm~\ref{alg:sign}, we search over the summary
    vectors $\vx_i$ one by one, until we find a summary vector $\vx_h$ that is not
    orthogonal to $\vu$. However, in Algorithm~\ref{alg:basis}, we just assumed that
    $\vx_1$ is not orthogonal to $V$, $\vx_2$ is not orthogonal to $\langle\vu_1\rangle
    ^\perp$, $\vx_3$ is not orthogonal to $\langle\vu_1,\vu_2\rangle^\perp$,
    \textit{etc}. We did not search for non-orthogonal summary vectors; we just
    assumed that they are. This observation is more obvious when we look at the
    special case of Algorithm~\ref{alg:basis} when $d=1$. In this case,
    Assumption~\ref{ass:k} is always satisfied, and Assumption~\ref{ass:perp}
    requires that $\vx_1$ is not orthogonal to $\vu$, which is stricter than
    Assumption~\ref{ass:mild} in our MAP-sign algorithm. This raises the question of
    whether we can strengthen Algorithm~\ref{alg:basis} such that it is as
    powerful as Algorithm~\ref{alg:sign} when taking $d=1$.

    This can be achieved by searching for non-orthogonal summary vectors at each
    step just as in Algorithm~\ref{alg:sign}. Suppose we have obtained the summary
    vectors $\vx_i,i=1,\dots,k$. First we search $\{\vx_i\}$ for the first summary
    vector $\vx_1^*$ that is not orthogonal to $V$, and choose $\vu_1\in V$ that
    maximizes $\vu_1^\top\vx_1^*$. Next we search $\{\vx_i\}$ for the first summary
    vector $\vx_2^*$ that is not orthogonal to $\langle\vu_1\rangle^\perp$, and
    choose $\vu_2\in\langle\vu_1\rangle^\perp$ that maximizes $\vu_2^\top\vx_2^*$.
    Then we search $\{\vx_i\}$ for the first summary vector $\vx_3^*$ that is not
    orthogonal to $\langle\vu_1,\vu_2\rangle^\perp$, and choose $\vu_3\in\langle\vu_1,
    \vu_2\rangle^\perp$ that maximizes $\vu_3^\top\vx_3^*$, and so on. The complete
    workflow is shown in Algorithm~\ref{alg:further-basis}.

    \begin{algorithm}[htbp]
        \caption{A stronger algorithm for eliminating basis ambiguity}
        \begin{algorithmic}
            \Require Eigenvalue $\lambda$ with multiplicity $d>1$
            \Ensure Spectral embedding corresponding to $\lambda$
            \State Calculate the eigenvectors $\mU\in\R^{n\times d}$ of $\lambda$ through eigendecomposition
            \State $\mP\gets\mU\mU^\top$
            \State $\alpha_i\gets|\mP\ve_i|,1\leq i\leq n$
            \State $k\gets$ the number of distinct values in $\{\alpha_i\}$
            \State Divide all basis vectors $\{\ve_i\}$ into $k$ disjoint groups $\mathcal{B}_i$
            according to the values of $\{\alpha_i\}$
            \State $\vx_i\gets\sum_{\ve_j\in\mathcal{B}_i}\ve_j+c\vj,i=1,\dots,k$
            \For{$i=1,2,\dots,d$}
                \State $\vx_i^*\gets$ the first summary vector in $\{\vx_i\}$ not
                perpendicular to $\langle\vu_1,\dots,\vu_{i-1}\rangle^\perp$
                \State Choose $\vu_i\in\langle\vu_1,\dots,\vu_{i-1}\rangle^\perp$,
                $|\vu_i|=1$, \textit{s.t.}\ $f(\vu_i)=\vu_i^\top\vx_i^*$ is maximized
            \EndFor
            \State \Return{$\mU_0\coloneqq[\vu_1,\dots,\vu_d]$}
        \end{algorithmic}
        \label{alg:further-basis}
    \end{algorithm}

    In order for Algorithm~\ref{alg:further-basis} to succeed, we require that
    such non-orthogonal summary vector $\vx_i^*$ exists at each step, which is
    less restrictive than Assumption~\ref{ass:k} and Assumption~\ref{ass:perp}.
    Thus we obtain a stronger algorithm for eliminating basis ambiguity.

    \section{An alternative approach to dealing with sign ambiguity}

    In our design for algorithms that eliminate sign ambiguity, we find one type of
    eigenvectors especially difficult to deal with. That is, uncanonizable eigenvectors
    $\vu\in\R^n$ such that there exists a permutation matrix $\mP\in\R^{n\times n}$
    satisfying $\vu=-\mP\vu$. This means that $\vu$ and $-\vu$ only differ by a
    permutation. Since they are uncanonizable by Corollary~\ref{cor:sign-canonizable},
    none of the solutions mentioned in Appendix~\ref{app:assumption1} can handle
    such eigenvectors. However, depending on the model architecture, such eigenvectors
    may not cause ambiguities at all. For example, consider the DeepSets-like
    architecture $\rho(\sum\phi(\vu_i))$, where $\phi$ is permutation-invariant.
    Then since $\vu$ and $-\vu$ only differ by a permutation, they produce the same
    output when fed to a permutation-invariant network, thus no ambiguities exist.
    This shows we can delay the handling of uncanonizable eigenvectors to the
    training stage, though it may result in loss of expressive power.

    On the other hand, for eigenvectors that are canonizable, we already know that
    Algorithm~\ref{alg:sign} canonizes them. Here we give an equivalent algorithm
    for removing sign ambiguity, shown in Algorithm~\ref{alg:alternative}.

    \begin{algorithm}[htbp]
        \caption{An alternative approach to dealing with sign ambiguity}
        \begin{algorithmic}
            \Require Input graph $\gG=(\sV,\sE,\mX)$
            \Ensure Spectral embedding of $\gG$
            \State Calculate the eigendecomposition $\hat{\mA}=\mU\mLambda\,\mU^\top$
            \For{each eigenvector $\vu\in\R^n$}
                \State $h\gets$ the smallest positive odd integer such that
                $\sum_{i=1}^n\evu_i^h\neq 0$
                \State Substitute $\vu$ with $-\vu$ if $\sum_{i=1}^n\evu_i^h<0$
            \EndFor
            \State \Return{$\mU$}
        \end{algorithmic}
        \label{alg:alternative}
    \end{algorithm}

    \begin{theorem}\label{thm:alternative}
        Algorithm~\ref{alg:alternative} uniquely decides the signs of canonizable
        eigenvectors and is permutation-equivariant.
    \end{theorem}

    Theorem~\ref{thm:alternative} is proved in Appendix~\ref{app:alternative}.

    Algorithm~\ref{alg:alternative} is well-motivated and better to understand than
    Algorithm~\ref{alg:sign} in some sense. Consider a na\"ive canonization of
    $\vu$ where we choose the sign such that $\vu$ has positive sum. This canonization
    algorithm is quite simple, but it cannot canonize all canonizable eigenvectors,
    such as the ones with zero sum. What Algorithm~\ref{alg:alternative} does is
    to go on to look at the sum of the 3rd power of $\vu$ and, if it is non-zero,
    choose the sign such that it is positive. If unfortunately it is zero, we
    go on to look at the sum of the 5th power and so on. It can be proved that
    there must exists a positive odd integer $h\leq n$ such that the sum of the
    $h$-th power of $\vu$ is non-zero. Thus Algorithm~\ref{alg:alternative}
    terminates within $\frac{n+1}2$ steps, successfully canonizing all canonizable
    eigenvectors.

    Algorithm~\ref{alg:alternative} offers an alternative approach to dealing with
    sign ambiguity in addition to Algorithm~\ref{alg:sign}, though it cannot
    generalize to the basis ambiguity case. The time complexity of
    Algorithm~\ref{alg:alternative} is $\mathcal{O}(n^2\log n)$.

    \section{Proofs}\label{app:proofs}

    \subsection{Proof of Theorem~\ref{thm:se-universal}}\label{app:se-universal}

    We first prove that $\mU\mLambda^\frac12$ is a real-valued matrix.

    \begin{lemma}\label{lem:3}
        Suppose $\hat{\mA}$ is the normalized adjacency matrix of a graph $\gG$,
        and $\hat{\mA}=\mU\mLambda\,\mU^\top$ is its spectral decomposition.
        Then, $\mU\mLambda^\frac12\in\R^{n\times n}$.
    \end{lemma}

    \begin{proof}
        Let $\mLambda=\diag(\vlambda)$. It suffices to show that $\evlambda_i\geq 0$
        for $i=1,2,\dots,n$.

        Let $\gG=(\sV,\sE)$, where $\sV$ is the vertex set and $\sE$ is the edge set.
        For node $i$, we denote the degree of node $i$ by $d_i$. For any $\vx\in\R^n$,
        we have
        \[
            \vx^\top\hat{\mA}\vx=\vx^\top(\mI+\tilde{\mA})\vx=\sum_
            {i\in\sV}\evx_i^2+\sum_{(i,j)\in\sE}\frac{2\evx_i\evx_j}{\sqrt{d_id_j}}=
            \sum_{(i,j)\in\sE}\left(\frac{\evx_i}{\sqrt{d_i}}+\frac{\evx_j}{\sqrt{d_j}}
            \right)^2\geq 0,
        \]
        thus the Rayleigh quotient of $\hat{\mA}$ is bounded by
        $
            \frac{\vx^\top\hat{\mA}\vx}{\vx^\top\vx}\geq 0.
        $
        The Rayleigh quotient gives the lower bound of eigenvalues of $\hat{\mA}$,
        therefore we have $\evlambda_i\geq 0$, and this completes the proof.
    \end{proof}

    Then we give the proof of Theorem~\ref{thm:se-universal}.

    \begin{proof}
        We can rewrite $f$ such that it is a continuous set invariant function on the
        set consisting of the rows of its input:
        \[
            f([\mX,\hat{\mA}])=f([\mX,\mU\mLambda\,\mU^\top])
            =f\left(\left[\mX,(\mU\mLambda^\frac12)
            (\mU\mLambda^\frac12)^\top\right]\right)
            =F([\mX,\mU\mLambda^\frac12]).
        \]

        Using the permutation invariance property of $f$, we can verify that $F$ is set
        invariant by observing:
        \begin{multline*}
            F([\mP\mX,\mP\mU\mLambda^\frac12])=f\left(\left[\mP\mX,(\mP\mU
            \mLambda^\frac12)(\mP\mU\mLambda^\frac12)^\top
            \right]\right)\\
            =f([\mP\mX,\mP\hat{\mA}\mP^\top])
            =f([\mX,\hat{\mA}])
            =F([\mX,\mU\mLambda^\frac12]).
        \end{multline*}

        Thus a universal network on sets can approximate $F$ to an arbitrary precision.
    \end{proof}

    \subsection{Proof of Theorem~\ref{thm:canonizable}}\label{app:canonizable}

    \begin{proof}
        On the one hand, if $f(hx)=gf(x)$ holds for some $h\in H$ and $g\in G$, then
        we have $\mathcal{A}\bigl(f(hx)\bigr)=h\mathcal{A}\bigl(f(x)\bigr)$ by the
        equivariance property of $\mathcal{A}$ and $\mathcal{A}\bigl(gf(x)\bigr)=
        \mathcal{A}\bigl(f(x)\bigr)$ by the invariance property of $\mathcal{A}$.
        However, since $\mathcal{A}$ is a \emph{mapping} and $f(hx)=gf(x)$, there
        must be $\mathcal{A}\bigl(f(hx)\bigr)=\mathcal{A}\bigl(gf(x)\bigr)$ and thus
        $\mathcal{A}\bigl(f(x)\bigr)=h\mathcal{A}\bigl(f(x)\bigr)$. Now we have $x\neq
        hx$ and $\mathcal{A}\bigl(f(x)\bigr)=h\mathcal{A}\bigl(f(x)\bigr)$,
        contradicting the universality property of $\mathcal{A}$.

        On the other hand, if there does not exist $h\in H$ and $g\in G$ such that
        $x\neq hx$ and $f(hx)=gf(x)$, we can construct a canonization of $x$ as
        follows. First arbitrarily choose $y_0$ such that $f(x)=y_0$, and let
        $\mathcal{A}\bigl(f(x)\bigr)=y_0$. For any $h\in H$ such that $x\neq hx$,
        by the equivariance of $f$ we know that $f(hx)=hy_0$, thus we let $\mathcal{A}
        \bigl(f(hx)\bigr)=hy_0$. Since $f(hx)\neq gf(x)$ for any $g$, we know that
        $y_0\neq hy_0$, thus the universality property of $\mathcal{A}$ holds.
        $\mathcal{A}$ is also invariant, since $\mathcal{A}\bigl(f(x)\bigr)$ is
        uniquely determined; and equivariant, since $\mathcal{A}\bigl(f(hx)\bigr)=
        hy_0=h\mathcal{A}\bigl(f(x)\bigr)$ for any $h$. We can repeat this process
        for all equivalence classes in $\mathcal{X}$ and obtain an invariant,
        equivariant and universal canonization for all inputs.
    \end{proof}

    \subsection{Proof of Corollary~\ref{cor:sign-canonizable}}
    \label{app:sign-canonizable}

    \begin{proof}
        Under sign ambiguity, we have $H=S_n$ and $G=\{+1,-1\}$. By
        Theorem~\ref{thm:canonizable}, $\vu$ is canonizable if and only if
        there does not exist a permutation $\sigma$ such that $\vu\neq\sigma(\vu)$
        and $\vu=\pm\sigma(\vu)$. This is equivalent to say that there does not
        exist a permutation matrix $\mP\in\R^{n\times n}$ such that $\vu=-\mP\vu$.
    \end{proof}

    \subsection{Proof of Corollary~\ref{cor:basis-canonizable}}
    \label{app:basis-canonizable}

    \begin{proof}
        Under basis ambiguity, we have $H=S_n$ and $G=O(d)$. By
        Theorem~\ref{thm:canonizable}, $\mU$ is canonizable if and only if
        there does not exist a permutation $\sigma$ and an orthonormal matrix
        $\mQ\in O(d)$ such that $\mU\neq\sigma(\mU)$ and $\mU=\sigma(\mU)\mQ$.
        This is equivalent to say that there does not exist a permutation matrix
        $\mP\in\R^{n\times n}$ such that $\mU\neq\mP\mU$ and $\mU$ and $\mP\mU$ span
        the same $d$-dimensional subspace $V\subseteq\R^n$. Note here we used a lemma
        from \citet{signnet}: for any orthonormal bases $\mV$ and $\mW$ of the same
        subspace, there exists an orthogonal $\mQ\in O(d)$ such that $\mV\mQ=\mW$.
    \end{proof}

    \subsection{Proof of Theorem~\ref{thm:sign}}\label{app:sign}

    \begin{proof}
        Without loss of generality, we can always assume that the angles $\{\alpha_i\}$
        are \emph{sorted} (if they are not, we can simply rearrange $\ve_i$ to make
        them sorted and proceed in the same way):
        \begin{multline*}
            |\mP\ve_1|=\dots=|\mP\ve_{n_1}|>|\mP\ve_{n_1+1}|=\dots=|\mP\ve_{n_1+n_2}|>
            \cdots\\
            >|\mP\ve_{n_1+\dots+n_{k-1}+1}|=\dots=|\mP\ve_{n_1+\dots+n_k}|,
        \end{multline*}
        where $k$ is the number of distinct lengths of $\mP\ve_i$, $\sum_{i=1}^kn_i=n$.
        Here we divide $\ve_i$ into $k$ groups according to the angles between them
        and the eigenspace, with each group sharing the same $|\mP\ve_i|$. Define
        $\vj=(1,1,\dots,1)\in\R^n$, then $\vx_i$ can be expressed as
        \[
            \vx_i=\ve_{n_1+\dots+n_{i-1}+1}+\dots+\ve_{n_1+\dots+n_i}+c\vj,\ 1\leq i
            \leq k.
        \]

        The sign of $\vu_0$ is selected based on the sign of $\vu^\top\vx_h(\neq 0)$,
        which is sign-equivariant. No matter what the sign of $\vu$ is, we will
        always choose the one that maximizes $\vu^\top\vx_h$, thus $\vu_0$ is
        sign-invariant.

        Suppose the entries of the input eigenvector $\vu$ is permutated by $\sigma\in
        S_n$, where $S_n$ is the permutation group of order $n$. Then we have $\evu_i'
        =\evu_{\sigma(i)}(1\leq i\leq n)$ and
        \begin{multline*}
            |\vu'\vu'{}^\top\ve_{\sigma(1)}|=\dots=|\vu'\vu'{}^\top\ve_{\sigma
            (n_1)}|>|\vu'\vu'{}^\top\ve_{\sigma(n_1+1)}|=\dots=|\vu'\vu'{}^\top
            \ve_{\sigma(n_1+n_2)}|>\cdots\\
            >|\vu'\vu'{}^\top\ve_{\sigma(n_1+\dots+n_{k-1}+1)}|=\dots=|\vu'\vu'{}^
            \top\ve_{\sigma(n_1+\dots+n_k)}|,
        \end{multline*}
        thus $(\evx_i')_i=(\evx_i)_{\sigma(i)}$, \textit{i.e.}, the vectors $\vx_i
        (1\leq i\leq k)$ are permutation-equivariant. Let $\vx_h'$ be defined as in
        Section~\ref{sec:sign} (the number $h$ is permutation-invariant because each
        $\vu'{}^\top\vx_i'$ is permutation-invariant). The sign of $\vu_0'$ is
        determined by the sign of the dot product of $\vu'$ and $\vx_h'$, both of
        which are permutation-equivariant. This indicates that the sign of $\vu'{}^
        \top\vx_h'$ (and thus the sign of $\vu_0'$) is permutation-invariant
        (\textit{i.e.}, unique). Since $\vu_0'=\vu'$ or $\vu_0'=-\vu'$, $\vu_0'$ is
        permutation-equivariant as well.

        If there exists a permutation $\sigma$ (acting on entries of $\vu$) such
        that $\vu\neq\sigma(\vu)$ but they have the same canonical form, then either
        $\vu=+\sigma(\vu)$ or $\vu=-\sigma(\vu)$. The former is impossible since we
        already assumed they are not equal, and the latter violates
        Assumption~\ref{ass:mild}, leading to a contradiction. Thus the canonization
        of $\vu$ is universal.
    \end{proof}

    \subsection{Proof of Theorem~\ref{thm:up-to-permutation}}
    \label{app:up-to-permutation}

    \begin{proof}
        Without loss of generality, we can always assume that the angles $\{\alpha_i\}$
        are \emph{sorted} (if they are not, we can simply rearrange $\ve_i$ to make
        them sorted and proceed in the same way):
        \begin{multline*}
            |\mP\ve_1|=\dots=|\mP\ve_{n_1}|>|\mP\ve_{n_1+1}|=\dots=|\mP\ve_{n_1+n_2}|>
            \cdots\\
            >|\mP\ve_{n_1+\dots+n_{k-1}+1}|=\dots=|\mP\ve_{n_1+\dots+n_k}|,
        \end{multline*}
        where $k$ is the number of distinct lengths of $\mP\ve_i$, $\sum_{i=1}^kn_i=n$.
        Here we divide $\ve_i$ into $k$ groups according to the angles between them
        and the eigenspace, with each group sharing the same $|\mP\ve_i|$. Define
        $\vj=(1,1,\dots,1)\in\R^n$, then $\vx_i$ can be expressed as
        \[
            \vx_i=\ve_{n_1+\dots+n_{i-1}+1}+\dots+\ve_{n_1+\dots+n_i}+c\vj,\ 1\leq i
            \leq k.
        \]

        Notice that $|\mP\ve_i|=|\vu\vu^\top\ve_i|=\evu_i|\vu|$, thus
        \[
            |\evu_1|=\dots=|\evu_{n_1}|>\dots>|\evu_{n_1+\dots+n_{k-1}+1}|=\dots
            =|\evu_{n_1+\dots+n_k}|.
        \]

        Suppose $\vu$ violates Assumption~\ref{ass:mild}. Then for any $1\leq j\leq k$,
        we have
        \[
            \evu_{n_1+\dots+n_{j-1}+1}+\dots+\evu_{n_1+\dots+n_j}=0.
        \]
        Let $\tilde{\vu}_j\coloneqq(\evu_{n_1+\dots+n_{j-1}+1},\dots,\evu_{n_1+\dots+
        n_j})^\top$. The above equations show that (1) the absolute value of
        the entries of $\tilde{\vu}_j$ are all equal; (2) the sum of the entries of
        $\tilde{\vu}_j$ is $0$. Thus for any entry of $\tilde{\vu}_j$, either it is
        $0$ or its positive and negative value appears in pairs. In conclusion,
        $+\tilde{\vu}_j$ and $-\tilde{\vu}_j$ are equal up to a permutation for all
        $j$, thus $+\vu$ and $-\vu$ are equal up to a permutation. By
        Theorem~\ref{cor:sign-canonizable}, it is not canonizable.

        On the other hand, suppose $\vu$ satisfies Assumption~\ref{ass:mild}.
        Then there exists $1\leq j\leq k$ such that
        \[
            \evu_{n_1+\dots+n_{j-1}+1}+\dots+\evu_{n_1+\dots+n_j}\neq 0.
        \]
        Let $\tilde{\vu}_j\coloneqq(\evu_{n_1+\dots+n_{j-1}+1},\dots,\evu_{n_1+\dots+
        n_j})^\top$. $\tilde{\vu}_j$ contains all entries of $\vu$ with absolute
        value $|\evu_{n_1+\dots+n_j}|$. Since the sum of $\tilde{\vu}_j$ is non-zero,
        the numbers of positive and negative entries in $\tilde{\vu}_j$ are
        different. No matter how we permute $+\vu$ and $-\vu$, their corresponding
        entries in $\tilde{\vu}_j$ will not align. Thus $+\vu$ and $-\vu$ are not
        equal up to any permutation. By Theorem~\ref{cor:sign-canonizable},
        it is canonizable.
    \end{proof}

    \subsection{Proof of Theorem~\ref{thm:basis}}\label{app:basis}

    We first point out that the orthogonal projector $\mP=\mU\mU^\top$ is
    invariant to the choice of basis.

    \begin{lemma}\label{lem:P}
        Let $\mU=[\vu_1,\dots,\vu_d]\in\R^{n\times d}$ and $\mV=[\vv_1,\dots,\vv_d]
        \in\R^{n\times d}$ be two sets of orthonormal vectors that span the same
        $d$-dimensional subspace $V\subseteq\R^n$. Then $\mU\mU^\top=\mV\mV^
        \top$.
    \end{lemma}

    \begin{proof}
        We have $\mU^\top\mU=\mV^\top\mV=\mI$ by the definition of $\mU$
        and $\mV$. Let $\mU=\mV\mQ$ where $\mQ$ is an invertible matrix, then
        \begin{align*}
            \mU\mU^\top&=\mU(\mU^\top\mU)^{-1}\mU^\top\\
            &=\mV\mQ(\mQ^\top\mV^\top\mV\mQ)^{-1}\mQ^\top\mV^
            \top\\
            &=\mV\mQ\mQ^{-1}(\mV^\top\mV)^{-1}(\mQ^\top)^{-1}\mQ^
            \top\mV^\top=\mV\mV^\top.
        \end{align*}
    \end{proof}

    Lemma~\ref{lem:P} shows that $\mP\ve_i$ (and thus $\vx_i$) is invariant to
    the choice of basis in $V$. If we permute $\ve_1,\dots,\ve_n$ by $\sigma\in S_n$,
    then the sorted sequence of $|\mP\ve_i|$ is also permuted by $\sigma$. Thus
    the elements of each $\vx_i$ is permuted by $\sigma$ as well. This shows the
    choice of $\vx_i$ is \emph{permutation-equivariant}.

    Then we prove Theorem~\ref{thm:basis}.

    \begin{proof}
        Without loss of generality, we can always assume that the angles $\{\alpha_i\}$
        are \emph{sorted} (if they are not, we can simply rearrange $\ve_i$ to make
        them sorted and proceed in the same way):
        \begin{multline*}
            |\mP\ve_1|=\dots=|\mP\ve_{n_1}|>|\mP\ve_{n_1+1}|=\dots=|\mP\ve_{n_1+n_2}|>
            \cdots\\
            >|\mP\ve_{n_1+\dots+n_{k-1}+1}|=\dots=|\mP\ve_{n_1+\dots+n_k}|,
        \end{multline*}
        where $k$ is the number of distinct lengths of $\mP\ve_i$, $\sum_{i=1}^kn_i=n$.
        Here we divide $\ve_i$ into $k$ groups according to the angles between them
        and the eigenspace, with each group sharing the same $|\mP\ve_i|$. Define
        $\vj=(1,1,\dots,1)\in\R^n$, then $\vx_i$ can be expressed as
        \[
            \vx_i=\ve_{n_1+\dots+n_{i-1}+1}+\dots+\ve_{n_1+\dots+n_i}+c\vj,\ 1\leq i
            \leq k.
        \]

        We have already shown the existence of the maximum value of $f(\vu)$.
        To show that basis-invariance of $\mU_0$, it suffices to show the uniqueness of
        the maximum point of $f(\vu)$. Thus no matter what basis of $\mU$ is,
        Algorithm~\ref{alg:basis} always yields the same output.

        Notice that $f(-\vu)=-f(\vu)$, thus either the maximum value of $f(\vu)$
        is positive, or $f(\vu)=\vu^\top\vx_i\equiv0$. However, $f(\vu)\equiv0$
        implies that $\vx_i$ is perpendicular to $\langle\vu_1,\dots,\vu_{i-1}\rangle
        ^\perp$, violating Assumption~\ref{ass:perp}. Thus we conclude the maximum
        value of $f(\vu)$ is positive.

        Suppose there exists $\vu'\neq\vu''$ such that $f(\vu')=f(\vu'')$ takes
        maximum value. Consider the vector $\alpha\vu'+\alpha\vu''$ where
        $\alpha=\sqrt{\frac1{2(1+\vu'{}^\top\vu'')}}>\frac12$. Obviously,
        $\alpha\vu'+\alpha\vu''\in\langle\vu_1,\dots,\vu_{i-1}\rangle^\perp$,
        $|\alpha\vu'+\alpha\vu''|=1$, and $f(\alpha\vu'+\alpha\vu'')=2\alpha\vu'{}^
        \top\vx_i>\vu'{}^\top\vx_i$. This leads to a contradiction.
        Therefore, the choice of $\vu_i$ is unique.

        The permutation-equivariance of Algorithm~\ref{alg:basis} can be proved by
        observing that each step of Algorithm~\ref{alg:basis} is
        permutation-equivariant. Since each $\vx_i$ is permutation-equivariant, its
        eigenprojection on the subspace $V$ (and thus $\mU_0$) is also
        permutation-equivariant.

        If there exists a permutation $\sigma$ (acting on rows of $\mU$) such that
        $\mU\neq\sigma(\mU)$ but they have the same canonical form, then $\mU$ and
        $\sigma(\mU)$ spans the same subspace. On the one hand, $\mU\neq\sigma(\mU)$
        means that at least one of the eigenvectors in $\mU$ is not $\sigma$-invariant.
        On the other hand, since all $\vu_1,\dots,\vu_d$ are permutation-equivariant
        but unchanged after $\sigma$, they are all $\sigma$-invariant. This leads to
        a contradiction, since it is impossible to have a non-$\sigma$-invariant
        eigenvector in a $\sigma$-invariant eigenspace.
    \end{proof}

    \subsection{Proof of Theorem~\ref{thm:linear-rni-universal}}\label{app:linear-rni-universal}

    We first prove the following lemmas.

    \begin{lemma}\label{lem:1}
        Let\/ $\rmR\in\R^{n\times n}$ be a random matrix, and each entry of\/ $\rmR$
        is sampled independently from the standard Gaussian distribution $N(0,1)$.
        Then with probability $1$, $\rmR$ has full rank.
    \end{lemma}

    \begin{proof}
        We denote the first column of $\rmR$ by $\rmR_{:,1}$. It is linearly independent
        because $\rmR_{:,1}=\vzero$ with probability $1$. Then we view $\rmR_{:,1}$ as
        fixed, and consider the second column $\rmR_{:,2}$. The probability that
        $\rmR_{:,2}$ falls into the span of $\rmR_{:,1}$ is $0$, thus with probability
        $1$, $\rmR_{:,1}$ and $\rmR_{:,2}$ are linearly independent.

        Generally, let us consider the $k$-th column $\rmR_{:,k}$. The first $k-1$
        columns of $\rmR$ forms a subspace in $\R^n$ whose Lebesgue
        measure is $0$. Thus $\rmR_{:,k}$ falls into this subspace with probability
        $0$. By inference, we have all the columns of $\rmR$ are linearly independent
        with probability $1$, \textit{i.e.}, $P(\rank(\rmR)=n)=1$.
    \end{proof}

    \begin{lemma}\label{lem:2}
        Let $\mA\in\R^{s\times n}$, $\mB\in\R^{s\times m}$ be two matrices. Then the
        equation $\mA\mX=\mB$ has a solution iff.\ $\rank(\mA)=\rank([\mA,\mB])$.
    \end{lemma}

    \begin{proof}
        First we prove the necessity. Suppose $\mA\mX=\mB$ has a solution. Then,
        \[
            [\mA_{:,1},\mA_{:,2},\dots,\mA_{:,n}]\mX_{:,i}=\mB_{:,i},
        \]
        where $\mM_{:,i}$ denotes the $i$-th column of matrix $\mM$. This means
        each column of $\mB$ can be expressed as a linear combination of the columns
        of $\mA$, and therefore each column of $[\mA,\mB]$ can be expressed as a
        linear combination of the columns of $\mA$.

        On the other hand, it is obvious that each column of $\mA$ can be expressed
        as a linear combination of the columns of $[\mA,\mB]$. Thus we have
        $\rank(\mA)=\rank([\mA,\mB])$.

        Then we prove the sufficiency. Since $\rank(\mA)=\rank([\mA,\mB])$, and
        each column of $\mA$ can be expressed as a linear combination of columns
        of $[\mA,\mB]$, we have the columns of $\mA$ and the columns of $[\mA,\mB]$
        are equivalent. Therefore, each column of $\mB$ can be expressed as a
        linear combination of columns of $\mA$, \textit{i.e.}, $\mA\vx=\mB_{:,i}$ has a
        solution for $i=1,\dots,m$. Thus the equation $\mA\mX=\mB$ has a solution.
    \end{proof}

    Then we give the proof of Theorem~\ref{thm:linear-rni-universal}.

    \begin{proof}
        For any prediction $\mZ\in\R^{n\times d'}$, we wish to prove that with
        probability $1$, there exists parameters of a linear GNN with RNI $\mW\in\R^
        {d\times d'}$ such that
        \begin{equation}\label{eq:eq}
            [\mX,\rmR]\mW=\mZ.
        \end{equation}

        By Lemma~\ref{lem:2}, the necessary and sufficient condition that \Eqref{eq:eq}
        has a solution $\mW$ is $\rank([\mX,\rmR])=\rank([\mX,\rmR,\mZ])$.

        By Lemma~\ref{lem:1}, with probability $1$, $\rank(\rmR)=n$, therefore
        $\rank([\mX,\rmR])=\rank([\mX,\rmR,\mZ])=n$.

        Thus, in conclusion, with probability $1$, there exists parameters of a linear
        GNN with RNI $\mW\in\R^{d\times d'}$ such that the GNN produces $\mZ$.

        We can also prove linear GNNs' equivariance by observing that for any
        permutation matrix $\mP\in\R^{n\times n}$,
        \[
            [\mP\mX,\rmR]\mW=\mP[\mX,\rmR]\mW=\mP\mZ,
        \]
        where we used $\mP\rmR=\rmR$ because each entry of $\mP\rmR$ is also sampled
        from the standard Gaussian matrix $N(0,1)$.
    \end{proof}

    \subsection{Proof of Theorem~\ref{thm:alternative}}\label{app:alternative}

    The following lemmas are used in our proof.

    \begin{lemma}[Newton's Identities]\label{lem:newton}
        Let $x_1,x_2,\dots,x_n$ be variables, denote for $k\geq 1$ by $P_k$
        the $k$-th power sum:
        \[
            P_k=\sum_{i=1}^nx_i^k=x_1^k+\dots+x_n^k,
        \]
        and for $k\geq 0$ denote by $e_k$ the elementary symmetric polynomial.
        Then we have
        \[
            P_k=(-1)^{k-1}ke_k+\sum_{i=1}^{k-1}(-1)^{k-1+i}e_{k-i}P_i,
        \]
        for all $n\geq 1$ and $n\geq k\geq 1$.
    \end{lemma}

    \begin{lemma}[Vieta's Formulas]\label{lem:vieta}
        Let $f(x)=a_nx^n+a_{n-1}x^{n-1}+\dots+a_1x+a_0$ be a polynomial of degree
        $n$, $e_k$ be the elementary symmetric polynomial. Then we have
        \[
            e_1=-\frac{a_{n-1}}{a_n},\quad e_2=\frac{a_{n-2}}{a_n},\quad\dots,\quad
            e_n=(-1)^n\frac{a_0}{a_n}.
        \]
    \end{lemma}

    We first show that $h$ exists.

    \begin{lemma}\label{lem:h}
        Let $\vu\in\R^n$. Assume that there does not exists a permutation matrix
        $\mP\in\R^{n\times n}$ such that $\vu=-\mP\vu$. Then there exists a
        positive odd integer $h\leq n$ such that $\sum_{i=1}^n\evu_i^h\neq 0$.
    \end{lemma}

    \begin{proof}
        Suppose the opposite holds, \textit{i.e.}, $P_i=0$ for all odd $1\leq i
        \leq n$, where $P_i$ is the $i$-th power sum of entries of $\vu$. Let
        $\evu_1,\evu_2,\dots,\evu_n$ be the roots of the polynomial
        $f(x)=a_nx^n+a_{n-1}x^{n-1}+\dots+a_1x+a_0$.

        If $n$ is even, by Lemma~\ref{lem:newton},
        \begin{align*}
            0&=P_1=e_1,\\
            0&=P_3=e_1P_2-e_2P_1+3e_3,\\
            0&=P_5=e_1P_4-e_2P_3+e_3P_2-e_4P_1+5e_5,\\
            &\vdotswithin{=}\\
            0&=P_{n-1}=e_1P_{n-2}-e_2P_{n-3}+\dots+(n-1)e_{n-1},
        \end{align*}
        which gives us $e_1=e_3=\dots=e_{n-1}=0$. Then, by Lemma~\ref{lem:vieta},
        we have $a_{n-1}=a_{n-3}=\dots=a_1=0$. This indicates that $f(x)=g(x^2)$
        for some polynomial $g(x)$. Similarly, if $n$ is odd, by Lemma~\ref{lem:newton},
        \begin{align*}
            0&=P_1=e_1,\\
            0&=P_3=e_1P_2-e_2P_1+3e_3,\\
            0&=P_5=e_1P_4-e_2P_3+e_3P_2-e_4P_1+5e_5,\\
            &\vdotswithin{=}\\
            0&=P_n=e_1P_{n-1}-e_2P_{n-2}+\dots+ne_n,
        \end{align*}
        which gives us $e_1=e_3=\dots=e_n=0$. Then, by Lemma~\ref{lem:vieta}, we
        have $a_{n-1}=a_{n-3}=\dots=a_0=0$. This indicates that $f(x)=xg(x^2)$
        for some polynomial $g(x)$. Either way, all $n$ roots of $f(x)$ are
        symmetric with respect to the $y$-axis. Then there must exist a permutation
        matrix such that $\vu=-\mP\vu$, leading to a contradiction. Thus
        Lemma~\ref{lem:h} holds.
    \end{proof}

    Then we prove Theorem~\ref{thm:alternative}.

    \begin{proof}
        By Lemma~\ref{lem:h}, we have shown the existence of $h$. Since flipping
        the sign of $\vu$ also flips the sign of $\sum_{i=1}^n\evu_i^h$ (because
        $h$ is odd), Algorithm~\ref{alg:alternative} uniquely decides the sign
        of $\vu$. Since the algorithm outputs either $\vu$ or $-\vu$, it is also
        permutation-equivariant.
    \end{proof}

    \subsection{Proof of Theorem~\ref{thm:bound}}\label{app:bound}

    We prove that under mild conditions, the loss of expressive power induced by
    truncating RSE can be upper bounded, as shown in the following theorem.

    \begin{theorem}\label{thm:bound}
        Let\/ $\Omega\subset\R^{n\times d}\times\R^{n\times n}$ be a compact set
        of graphs, $[\mX,\hat{\mA}]\in\Omega$. Let\/ $\NN$ be a universal neural
        network on sets. Given an invariant graph function $f$ defined over\/
        $\Omega$ that can be $\varepsilon$-approximated by an $L_p$-Lipschitz
        continuous function and arbitrary $\varepsilon>0$, for any integer
        $0<k\leq n$, there exist parameters of\/ $\NN$ such that for all
        graphs $[\mX,\hat{\mA}]\in\Omega$,
        \[
            \bigl\lvert f([\mX,\hat{\mA}])-\NN([\mX,(\mU\mLambda^\frac12)_{:,-k:},
            \vzero])\bigr\rvert<\sqrt{n-k}L_p\lambda_{n-k}+\varepsilon.
        \]
        Here the $L_p$-Lipschitz continuity of $f$ is defined using the Frobenius
        norm on the input domain, $0\leq\lambda_1\leq\dots\leq\lambda_n\leq 2$ are
        the eigenvalues of $\hat{\mA}$, $\vzero\in\R^{n\times(n-k)}$.
    \end{theorem}

    We can see from Theorem~\ref{thm:bound} that the upper bound of the loss of
    expressive power decreases when $k$ increases, and when $k=n$, the network
    becomes universal. We give its proof as follows.

    By Lemma~\ref{lem:3}, we know that $\lambda_i\geq 0$ for $i=1,2,\dots,n$.
    Next we prove that $\lambda_i\leq 2$.

    \begin{lemma}\label{lem:4}
        Suppose $\hat{\mA}$ is the normalized adjacency matrix of a graph $\gG$,
        and $\lambda_1<\dots<\lambda_n$ are its eigenvalues. Then $\lambda_i\leq 2$,
        for $i=1,2,\dots,n$.
    \end{lemma}

    \begin{proof}
        In the proof of Lemma~\ref{lem:3}, we proved $\vx^\top(\mI+\tilde{\mA})
        \vx\geq 0$. Similarly, we have
        \[
            \vx^\top(\mI-\tilde{\mA})\vx=\sum_{(i,j)\in\sE}\left(\frac{x_i}
            {\sqrt{d_i}}-\frac{x_j}{\sqrt{d_j}}\right)^2\geq 0.
        \]
        Thus,
        \[
            \vx^\top\hat{\mA}\vx=\vx^\top(-\mI+\tilde{\mA})\vx+2\vx^
            \top\vx\leq 2\vx^\top\vx.
        \]
        This shows that the Rayleigh quotient is bounded by $\frac{\vx^\top
        \hat{\mA}\vx}{\vx^\top\vx}\leq 2$, therefore $\lambda_i\leq 2$.
    \end{proof}

    Then we give the proof of Theorem~\ref{thm:bound}.

    \begin{proof}
        Let $0\leq\lambda_1\leq\dots\leq\lambda_n\leq 2$ be the eigenvalues of $\hat{\mA}$
        and $\vu_1,\dots,\vu_n$ be the corresponding eigenvectors. Then
        \[
            \hat{\mA}=\lambda_1\vu_1\vu_1^\top+\dots+\lambda_n\vu_n\vu_n^
            \top.
        \]
        We also define
        \[
            \hat{\mA}'\coloneqq\lambda_{n-k+1}\vu_{n-k+1}\vu_{n-k+1}^\top+\dots+
            \lambda_n\vu_n\vu_n^\top.
        \]

        By Theorem~\ref{thm:se-universal} and the assumptions in our theorem,
        we know that there exists a permutation-invariant network on sets such that
        \[
            \bigl\lvert F([\mX,(\mU\mLambda^\frac12)_{:,-k:},\vzero])-\NN([\mX,
            (\mU\mLambda^\frac12)_{:,-k:},\vzero])\bigr\rvert<\frac{\varepsilon}2.
        \]
        Since $f$ can be approximated by an $L_p$-Lipschitz continuous function, we have
        \begin{align*}
            \bigl\lvert f([\mX,\hat{\mA}])-F([\mX,(\mU\mLambda^\frac12)_{:,-k:},
            \vzero])\bigr\rvert&=\bigl\lvert f([\mX,\hat{\mA}])-f([\mX,\hat{\mA}'])
            \bigr\rvert\\
            &\leq L_p\bigl\lVert[\mX,\hat{\mA}]-[\mX,\hat{\mA}']\bigr\rVert_\mathrm{F}
            +\frac{\varepsilon}2\\
            &=L_p\bigl\lVert[\vzero,\lambda_1\vu_1\vu_1^\top+\dots+
            \lambda_{n-k}\vu_{n-k}\vu_{n-k}^\top]\bigr\rVert_\mathrm{F}+
            \frac{\varepsilon}2\\
            &=L_p\sqrt{\lambda_1^2+\dots+\lambda_{n-k}^2}+\frac{\varepsilon}2\\
            &\leq\sqrt{n-k}L_p\lambda_{n-k}+\frac{\varepsilon}2.
        \end{align*}
        Combining the two inequalities above gives us
        \[
            \bigl\lvert f([\mX,\hat{\mA}])-\NN([\mX,(\mU\mLambda^\frac12)_{:,-k:},
            \vzero])\bigr\rvert<\sqrt{n-k}L_p\lambda_{n-k}+\varepsilon.
        \]
    \end{proof}

    \section{Dataset details}\label{app:datasets}

    \textbf{ZINC} (MIT License) consists of 12K molecular graphs from the ZINC
    database of commercially available chemical compounds. These molecular graphs
    are between 9 and 37 nodes large. Each node represents a heavy atom (28
    possible atom types) and each edge represents a bond (3 possible types).
    The task is to regress constrained solubility (logP) of the molecule. The
    dataset comes with a predefined 10K/1K/1K train/validation/test split.

    \textbf{OGBG-MOLTOX21 and OGBG-MOLPCBA} (MIT License) are molecular property
    prediction datasets adopted by OGB from MoleculeNet. These datasets use a
    common node (atom) and edge (bond) featurization that represent chemophysical
    properties. OGBG-MOLTOX21 is a multi-mask binary graph classification dataset
    where a qualitative (active/inactive) binary label is predicted against 12
    different toxicity measurements for each molecular graph. OGBG-MOLPCBA is
    also a multi-task binary graph classification dataset from OGB where an
    active/inactive binary label is predicted for 128 bioassays.

    Details of the three datasets are summarized in Table~\ref{tab:details}.

    \begin{table}[htbp]
        \centering
        \caption{Details of the datasets.}
        \begin{tabular}{cccc}
            \toprule
            Dataset          & ZINC             & ogbg-moltox21    & ogbg-molpcba \\
            \midrule
            \#Graphs         & 12000            & 7831             & 437929 \\
            Avg \#Nodes      & 23.2             & 18.6             & 26.0 \\
            Avg \#Edges      & 24.9             & 19.3             & 28.1 \\
            Task Type        & Regression       & Binary Classification & Binary Classification \\
            Metric           & MAE              & ROC-AUC          & AP \\
            \bottomrule
        \end{tabular}
        \label{tab:details}
    \end{table}

    \section{Hyperparameter settings}\label{app:hyperparameters}

    \subsection{Real-world tasks}

    We evaluate the proposed MAP on three real-world datasets: ZINC, OGBG-MOLTOX21
    and OGBG-MOLPCBA, on a server with 6 NVIDIA 3080 Ti GPUs and 2 NVIDIA 1080 Ti
    GPUs. We consider 4 GNN architectures: GatedGCN, PNA, SAN and GraphiT, with 4
    different positional encodings: no PE, LapPE with random Sign, SignNet and MAP\@.
    We follow the same settings as \citet{lspe} for models with no PE or LapPE,
    and same settings as \citet{signnet} for models with SignNet or MAP. All
    baseline scores reported in Table~\ref{tab:zinc}, \ref{tab:moltox21} \&
    \ref{tab:molpcba} are taken from the original papers. As shown in
    \Figref{fig:settings}, for models with no PE or LapPE, the input features are
    directly fed into the base model; for models with SignNet, the eigenvectors are
    first processed by SignNet and then concatenated with the original node features
    as input to the base model; for models with MAP, the PEs are first processed by
    a normal GNN and then concatenated with the original node features as input to
    the base model. These settings align with the original papers.

    \begin{figure}[htbp]
        \centering
        \def\svgwidth{135mm}
        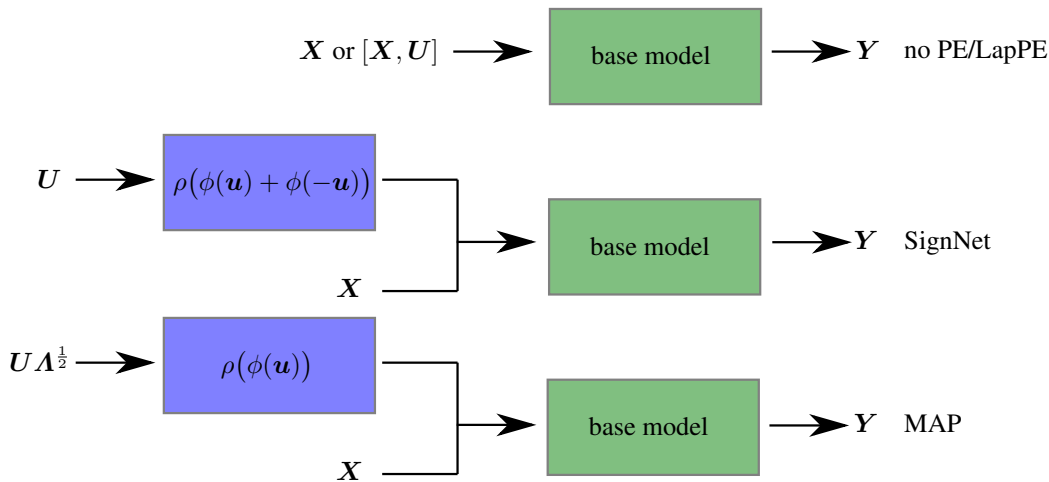
        \caption{Our experiment settings with different PEs.}
        \label{fig:settings}
    \end{figure}

    The main hyperparameters in our experiments are listed as follows.
    \begin{itemize}
        \item $k$: the number of eigenvectors used in the PE\@.
        \item $L_1$: the number of layers of the base model.
        \item $h_1$: the hidden dimension of the base model.
        \item $h_2$: the output dimension of the base model.
        \item $\lambda$: the initial learning rate.
        \item $t$: the patience of the learning rate schedular.
        \item $r$: the factor of the learning rate schedular.
        \item $\lambda_{\min}$: the minimum learning rate of the learning rate
        schedular.
        \item $L_2$: the number of layers of SignNet or the normal GNN (when using
        MAP as PE).
        \item $h_3$: the hidden dimension of SignNet or the normal GNN (when using
        MAP as PE).
    \end{itemize}

    The values of these hyperparameters in our experiments are listed in
    Table~\ref{tab:hyperparameters}.

    \begin{table}[htbp]
        \centering
        \caption{Hyperparameter details for experiments on real-world datasets.}
        \resizebox{\textwidth}{!}{
            \begin{tabular}{ccccccccccccc}
                \toprule
                            & Model          & PE             & $k$            & $L_1$          & $h_1$          & $h_2$          & $\lambda$      & $t$            & $r$            & $\lambda_{\min}$ & $L_2$          & $h_3$ \\
                \midrule
                \multirow{12}[2]{*}{\begin{sideways}ZINC\end{sideways}} & GatedGCN       & None           & 0              & 16             & 78             & 78             & 0.001          & 25             & 0.5            & 1e-6           & -              & - \\
                            & GatedGCN       & LapPE + RS     & 8              & 16             & 78             & 78             & 0.001          & 25             & 0.5            & 1e-6           & -              & - \\
                            & GatedGCN       & SignNet        & 8              & 16             & 67             & 67             & 0.001          & 25             & 0.5            & 1e-6           & 8              & 67 \\
                            & GatedGCN       & MAP            & 8              & 16             & 69             & 67             & 0.001          & 25             & 0.5            & 1e-5           & 6              & 69 \\
                            \cmidrule{2-13}
                            & PNA            & None           & 0              & 16             & 70             & 70             & 0.001          & 25             & 0.5            & 1e-6           & -              & - \\
                            & PNA            & LapPE + RS     & 8              & 16             & 80             & 80             & 0.001          & 25             & 0.5            & 1e-6           & -              & - \\
                            & PNA            & SignNet        & 8              & 16             & 70             & 70             & 0.001          & 25             & 0.5            & 1e-6           & 8              & 70 \\
                            & PNA            & MAP            & 8              & 16             & 70             & 70             & 0.001          & 25             & 0.5            & 1e-6           & 6              & 70 \\
                            \cmidrule{2-13}
                            & SAN            & None           & 0              & 10             & 64             & 64             & 0.0003         & 25             & 0.5            & 1e-6           & -              & - \\
                            & SAN            & MAP            & 16             & 10             & 40             & 40             & 0.0007         & 25             & 0.5            & 1e-5           & 6              & 40 \\
                            \cmidrule{2-13}
                            & GraphiT        & None           & 0              & 10             & 64             & 64             & 0.0003         & 25             & 0.5            & 1e-6           & -              & - \\
                            & GraphiT        & MAP            & 16             & 10             & 48             & 48             & 0.0007         & 25             & 0.5            & 1e-6           & 6              & 48 \\
                \midrule
                \multirow{9}[2]{*}{\begin{sideways}MOLTOX21\end{sideways}} & GatedGCN       & None           & 0              & 8              & 154            & 154            & 0.001          & 25             & 0.5            & 1e-5           & -              & - \\
                            & GatedGCN       & LapPE + RS     & 3              & 8              & 154            & 154            & 0.001          & 25             & 0.5            & 1e-5           & -              & - \\
                            & GatedGCN       & MAP            & 3              & 8              & 150            & 150            & 0.001          & 22             & 0.14           & 5e-6           & 8              & 150 \\
                            \cmidrule{2-13}
                            & PNA            & None           & 0              & 8              & 206            & 206            & 0.0005         & 10             & 0.8            & 2e-5           & -              & - \\
                            & PNA            & MAP            & 16             & 8              & 115            & 113            & 0.0005         & 10             & 0.8            & 8e-5           & 7              & 115 \\
                            \cmidrule{2-13}
                            & SAN            & None           & 0              & 10             & 88             & 88             & 0.0007         & 25             & 0.5            & 1e-6           & -              & - \\
                            & SAN            & MAP            & 12             & 10             & 88             & 88             & 0.0007         & 25             & 0.5            & 1e-5           & 8              & 88 \\
                            \cmidrule{2-13}
                            & GraphiT        & None           & 0              & 10             & 88             & 88             & 0.0007         & 25             & 0.5            & 1e-6           & -              & - \\
                            & GraphiT        & MAP            & 16             & 10             & 64             & 64             & 0.0007         & 25             & 0.5            & 1e-6           & 6              & 64 \\
                \midrule
                \multirow{5}[2]{*}{\begin{sideways}MOLPCBA\end{sideways}} & GatedGCN       & None           & 0              & 8              & 154            & 154            & 0.001          & 25             & 0.5            & 1e-4           & -              & - \\
                            & GatedGCN       & LapPE + RS     & 3              & 8              & 154            & 154            & 0.001          & 25             & 0.5            & 1e-4           & -              & - \\
                            & GatedGCN       & MAP            & 3              & 8              & 200            & 200            & 0.001          & 25             & 0.5            & 1e-5           & 8              & 200 \\
                            \cmidrule{2-13}
                            & PNA            & None           & 0              & 4              & 510            & 510            & 0.0005         & 4              & 0.8            & 2e-5           & -              & - \\
                            & PNA            & MAP            & 16             & 4              & 304            & 304            & 0.0005         & 10             & 0.8            & 2e-5           & 8              & 304 \\
                \bottomrule
            \end{tabular}
        }
        \label{tab:hyperparameters}
    \end{table}

    \subsection{Synthetic tasks}

    To verify the expressive power of RSE, we conduct experiments on the synthetic
    \textsc{Exp} dataset. The dataset consists of a set of 1-WL indistinguishable
    non-isomorphic graph pairs. If a network reaches above 50\,\% accuracy on this
    dataset, it must have expressive power beyond the 1-WL test. DeepSets-RSE
    is a two-layer DeepSets model with RSE as PE, whereas Linear-RSE is a one-layer
    linear model with RSE as PE\@. We use Optuna \citep{optuna} to optimize the
    hyperparameters of our models. The values of hyperparameters of our models are
    as follows:
    \begin{itemize}
        \item \textbf{DeepSets-RSE}: the learning rate $\lambda=0.002385602941230316$,
        the hidden dimension of the first linear layer $w_1=60$, the hidden dimension
        of the second linear layer $w_2=76$, the dropout rate \citep{dropout}
        $p=0.13592575703525184$, the weight decay of Adam optimizer $\mathit{wd}
        =0.0005$.
        \item \textbf{Linear-RSE}: the learning rate $\lambda=0.0006867736568978745$,
        the hidden dimension $w=109$, the weight decay of Adam optimizer
        $\mathit{wd}=0.0001$.
    \end{itemize}

\end{document}